\documentclass[acmsmall,screen]{acmart}
\AtBeginDocument{%
  }

\setcopyright{acmlicensed}
\copyrightyear{2018}
\acmYear{2018}
\acmDOI{XXXXXXX.XXXXXXX}

\acmJournal{JACM}
\acmVolume{37}
\acmNumber{4}
\acmArticle{111}
\acmMonth{8}

\usepackage{amsmath}
\usepackage{bbm}
\usepackage{enumerate}
\usepackage{booktabs}
\usepackage{graphicx}
\usepackage{algorithm}
\usepackage{algorithmic}
\usepackage{hyperref}[colorlinks, linkcolor=black]
\usepackage{latexsym}
\usepackage{times}
\usepackage{soul}
\usepackage{url}
\usepackage{amsthm}
\usepackage{booktabs}
\usepackage[switch]{lineno}
\usepackage{bbding}
\usepackage{graphics,color,multirow,subfigure}
\usepackage{makecell}
\usepackage{amsfonts}
\usepackage[english]{babel}
\newtheorem{theorem}{Theorem}[section]
\newtheorem{proposition}[theorem]{Proposition}
\usepackage{enumitem}

\begin{document}

\title{Quantifying the Noise of Structural Perturbations on Graph Adversarial Attacks}

\author{Junyuan Fang}
\affiliation{%
  \institution{Department of Electrical Engineering, City University of Hong Kong}
  \city{Hong Kong SAR}
  \country{China}}
\email{junyufang2-c@my.cityu.edu.hk}

\author{Han Yang}
\affiliation{%
  \institution{School of Computer Science and Engineering, Sun Yat-sen University}
  \city{Guangzhou}
  \country{China}}
\email{yangh396@mail2.sysu.edu.cn}

\author{Haixian Wen}
\affiliation{%
  \institution{School of Computer Science and Engineering, Sun Yat-sen University}
  \city{Guangzhou}
  \country{China}}
\email{wenhx6@mail2.sysu.edu.cn}

\author{Jiajing  Wu}
\authornote{Corresponding author.}
\affiliation{%
  \institution{School of Software Engineering, Sun Yat-sen University}
  \city{Zhuhai}
  \country{China}
}
\email{wujiajing@mail.sysu.edu.cn}

\author{Zibin Zheng}
\affiliation{%
  \institution{School of Software Engineering, Sun Yat-sen University}
  \city{Zhuhai}
  \country{China}
}
\email{zhzibin@mail.sysu.edu.cn}

\author{Chi K. Tse}
\affiliation{%
  \institution{Department of Electrical Engineering, City University of Hong Kong}
  \city{Hong Kong SAR}
  \country{China}}
\email{chitse@cityu.edu.hk}

\renewcommand{\shortauthors}{Fang et al.}

\begin{abstract}
    Graph neural networks have been widely utilized to solve graph-related tasks because of their strong learning power in utilizing the local information of neighbors. However, recent studies on graph adversarial attacks have proven that current graph neural networks are not robust against malicious attacks. Yet much of the existing work has focused on the optimization objective based on attack performance to obtain (near) optimal perturbations, but paid less attention to the strength quantification of each perturbation such as the injection of a particular node/link, which makes the choice of perturbations a black-box model that lacks interpretability. In this work, we propose the concept of noise to quantify the attack strength of each adversarial link. Furthermore, we propose three attack strategies based on the defined noise and classification margins in terms of single and multiple steps optimization. Extensive experiments conducted on benchmark datasets against three representative graph neural networks demonstrate the effectiveness of the proposed attack strategies. Particularly, we also investigate the preferred patterns of effective adversarial perturbations by analyzing the corresponding properties of the selected perturbation nodes.
\end{abstract}



\received{\today}

\maketitle

\section{Introduction}
Graphs, where nodes indicate the entities and links represent the corresponding relationships between entities, have been widely employed to characterize the interaction of various social systems. In recent years, due to the strong learning power of the neighborhood aggregation mechanism, graph neural networks (GNNs) have obtained great success on various graph-related downstream tasks, such as node classification, link prediction, and community detection \cite{wu2020comprehensive, zhou2020graph,fang2024gani,wu2020graph,jiang2021graph}.

	Although GNNs perform well in these graph-based tasks, recent studies \cite{sun2018adversarial, chen2020survey} pointed out that current GNNs are not robust to carefully designed attacks (i.e., adversarial attacks). In other words, GNNs may be easily deceived after being injected with some small perturbations into the data, such as modifying the node features, adding or removing the links, injecting fake nodes, etc. Therefore, a series of works focusing on graph adversarial attacks has been proposed to further evaluate the robustness of current GNNs from different perspectives \cite{zugner2018adversarial, wu2019adversarial, dai2018adversarial, chen2018fast,sun2019node, tao2021single}.

	Most existing adversarial attack methods focus on the optimization objective based on final attack performance by utilizing gradient information \cite{chen2018fast, li2021adversarial}, classification margins \cite{dai2018adversarial}, reinforcement learning (RL) algorithm \cite{dai2018adversarial, sun2019node}, etc., to obtain (near) optimal perturbations. Yet these methods pay less attention to the strength quantification of each perturbation, such as the injection of a particular node/link into the clean graph, which makes the choice of perturbations a black-box model that lacks interpretability. Therefore, we want to identify the contributions of adversarial links and further uncover the selection preference of the optimal perturbation of graph adversarial attacks.

	In this work, we first propose the concept of \emph{noise} to quantify the attack strength of each adversarial link, and then propose three simple yet effective attack methods based on this noise concept to further evaluate the robustness of current GNNs. Specifically, following the classic work \cite{dai2018adversarial} in the graph adversarial attacks, we focus on the scenario of targeted attacks on the node classification task, which means the attack goal is to make GNNs predict the targeted node to the wrong classes. Based on neighborhood aggregations of GNNs, we theoretically propose and discuss the noise  of different adversarial links. In addition, we formulate a reasonable noise function to quantify the noise of each adversarial link on the structural attacks. We then refine the adversarial links based on their rankings of noise value. According to the defined noise function and the traditional classification margins, we then propose three different attack strategies from the perspectives of both single step and multiple steps optimization, namely \underline{n}oise-based \underline{g}reedy \underline{a}ttacks (\textbf{NGA}), \underline{n}oise and \underline{m}argin-based \underline{a}ttacks (\textbf{NMA}), and \underline{n}oise and \underline{m}argin-based \underline{a}ttacks-\underline{b}oost (\textbf{NMAB}). Finally, we verify the attack performance of the proposed methods on the benchmark datasets against three representative GNNs. Moreover, we also measure the properties of the adversarial links from several perspectives, including their class, degree, homophily ratio \cite{pei2020geom, zhu2020beyond, fang2025contributes}, etc. The major contributions of this work are summarized as follows.

	\begin{enumerate}
		\item \textbf{\em New Concept}. We theoretically analyze the perturbation of adversarial links from the perspective of noise propagation based on the neighborhood aggregation mechanism of GNNs, and then present the concept of \emph{noise} to quantify the perturbation strength of each adversarial link on neighborhood aggregations.
		\item \textbf{\em New Methods}. Based on the defined noise function and classification margins of adversarial links, we then propose three simple yet effective attack strategies by considering both the single step and multiple steps optimizations.
		\item \textbf{\em Comprehensive Experiments}. We conducted extensive experiments on three benchmark datasets against representative GNNs to verify the superiority of the proposed methods. Moreover, we further analyze the properties and patterns of the chosen adversarial links from several perspectives. 
	\end{enumerate}

	\section{Related Work}\label{sec:rw}

	\subsection{Graph Neural Network Models}
	Graph neural networks (GNNs) have obtained great success on different graph-related tasks by utilizing the message passing mechanism. Through the neighborhood aggregations of GNNs, each node in the graph can utilize the information of its neighbors to further characterize itself. A series of well-known GNNs have been proposed in the past few years. For example, the graph convolution network (GCN) \cite{kipf2016semi} achieved the feature propagation via a first-order approximation of localized spectral filters on graphs. GraphSAGE \cite{hamilton2017inductive} utilized neighborhood sampling mechanism and proposed more advanced aggregators to improve the scalability of GNNs. Graph attention network (GAT) \cite{velivckovic2017graph} assigned different weights to different neighbors adaptively via a self-attention mechanism. Simplified graph convolution network (SGC) \cite{wu2019simplifying} further reduced the complexity of the original GCN by removing the non-linear activation functions in the middle aggregation layers.

	\subsection{Categories of Graph Adversarial Attacks}
	Despite the strong potential shown on various downstream tasks, recent GNNs tend to be vulnerable to the particularly designed perturbations, which are called graph adversarial attacks.
	Based on the claim from previous studies \cite{sun2018adversarial, chen2020survey}, graph adversarial attacks can be divided into graph modification attacks and node injection attacks from the perspective of attack operations. The former assumes that the attackers can directly modify the original graphs by changing features or flipping links, while the latter considers that the attackers prefer to inject some fake nodes into the original graph. Moreover, from the perspective of attack goals, graph adversarial attacks can be divided into targeted attacks and global attacks, indicating that the attackers will focus on attacking a single node and the overall performance of all nodes each time, respectively. In addition, based on different attack stages, graph adversarial attacks can also be divided into evasion attacks (i.e., test time) which evaluate the attack performance on the pre-trained GNNs, and poisoning attacks (i.e., training time) which evaluate the attack performance on the retrained GNNs. 
	
	As the main objective of this work is to investigate the specific difference of the adversarial links in influencing the neighborhood aggregation mechanism of GNNs, we will focus on the \textbf{{\em graph modification attacks}} and \textbf{{\em targeted evasion attacks}} on the node classification task. In other words, our goal is to make GNNs to predict the targeted node to the wrong class via structural perturbations in the testing phase.

	\subsection{Existing Methods of Graph Adversarial Attacks}
	Z$\ddot{\rm u}$gner {\em et al.} \cite{dai2018adversarial} first pointed out that traditional GNNs can be easily deceived through small unnoticeable perturbations by proposing NETTACK method by selecting the adversarial links causing the largest classification margins. After that, a bunch of works have been proposed to investigate the robustness of GNNs. For instance, Dai {\em et al.} \cite{dai2018adversarial} proposed GradArgmax via flipping the corresponding link with the largest magnitude of gradients. Chen {\em et al.} \cite{chen2018fast} presented another gradient-based attack method, FGA, by adding/removing the corresponding valid link with the largest absolute gradient value. Zhang {\em et al.} \cite{zhang2020cross} utilized the cross entropy to denote the similarity of different nodes, then heuristically connect/disconnect the corresponding node pairs. Li {\em et al.} \cite{li2021adversarial} proposed SGA to achieve the gradient-based attacks by leveraging subgraph to reduce the time and space of gradient calculation. Geisler {\em et al.} \cite{geisler2021robustness} developed an attack strategy by adopting the sparsity-aware first-order optimization attacks and a novel surrogate loss to improve the scalability of attacks. Zhu {\em et al.} \cite{zhu2024simple} introduced the partial graph attack strategy by adopting a hierarchical selection policy to select the vulnerable nodes as the attack targets and a cost-effective anchor-picking policy to pick the most promising anchor nodes for modifying edges. Recently, Alom {\em et al.} \cite{alomgottack} proposed a subtle and effective attack method, GOttack, by manipulating the graph orbit vector of each node. Besides, there are also a series of strong attack methods from other optimization perspectives, such as the reinforcement learning algorithms \cite{sun2019node}, meta-gradients \cite{zugner2019adversarial}, evolutionary algorithms \cite{chen2019ga}, etc. However, few of them have analyzed the specific difference between different adversarial links, especially for targeted attacks. Therefore, we want to bridge this gap and quantitatively evaluate the difference of structural perturbations, then propose powerful attack strategies from the perspective of noise propagation on neighborhood aggregations.

\begin{table}[t]
    \centering
    \caption{Main notations in this paper.}
  \begin{tabular}{ll}
    \toprule
    Notation          & Definition             \\ \midrule
    $G$               & Graph                  \\
    $V$               & Node set               \\
    $E$               & Link set               \\
    $X$, $x_{u}$      & Node feature matrix, feature vector of node $u$  \\
    $Y$    & Label set   \\
    $\mathcal{N}(u)$ & Neighbor set of node $u$ including itself         \\
    $f_{\theta}$  & GNN model \\
    $h_u^{(k)}$ & Representation of node $u$ in the $k$-th layer \\
    $\Delta$ & Attack budget, i.e., number of links to be added \\
    $\mathcal{L}_{\rm train}$ & Training loss of GNN models \\

    $\mathcal{L}_{\rm atk}$ & Attack loss of the adversaries \\
    $\widetilde{LN} (u, v)$ & Link noise for link $(u, v)$\\
    ${\rm DIS} (u, v)$ & Dissimilarity between nodes $u$ and $v$\\
    $\widetilde{\rm LN} (u, v)$ & Appropriate link noise for link $(u, v)$\\
    \bottomrule
    \end{tabular}
    \label{tab:notation}
\end{table}

	\section{Preliminaries}\label{sec:p}
	\subsection{Graph Neural Networks}
	A graph can be denoted as $G = (V, E, X)$ where $V = \{v_1,v_2,\cdots,v_n\}$ represents the node set with $n$ nodes and $E=\{e_1,e_2,\cdots,e_m\}$ indicates the link set with $m$ links. $X \in \mathbb{R}^{n \times d}$ represents the feature matrix where each node has a $d$-dimension feature vector. We summarize the main notations used in this paper in Table~\ref{tab:notation}.
    
    The details of a general semi-supervised node classification task \cite{kipf2016semi} are as follows. Assuming the $n$ nodes belong to the label set $Y = \{0,1,\cdots,C-1\}$ with $C$ different classes, we will first train a GNN model $f_{\theta}$ based on the labeled/training nodes $V_{\rm train} \subset V $ by minimizing a training loss $\mathcal{L}_{\rm train}$, and then utilize $f_{\theta^*}$ to predict the possible class $y_i$ of remaining/test nodes $\overline{V} = \{v_i | v_i \in V \setminus V_{\rm train}\}$. Specifically, a $k$-layer GCN can be formally defined as follows.
	
	\begin{equation}\label{eq:gcn}
		h_u^{(k)} = \sigma(W^{(k)} \cdot Agg(h_v^{(k-1)} | v \in \mathcal{N}(u)),
	\end{equation}
	where $h_u^{(k)}$ indicates the representation of node $u$ in the $k$-th layer, $\mathcal{N}(u)$ is the neighbors of node $u$ including self-loop (i.e., $u \in \mathcal{N}(u)$). $\sigma$ is the activation function such as ReLU in the middle layers or softmax in the last layer, and $W^{(k)}$ is the learning parameter in the $k$-th layer. Moreover, $Agg(\cdot)$ is the aggregation function that combines the representation of neighboring nodes in the prior layer, which is as follows.

	\begin{equation}\label{eq:agg}
		Agg(h_v^{(k-1)} | v \in \mathcal{N}(u)) = \sum_{v \in \mathcal{N}(u)} \frac{1}{\sqrt{|\mathcal{N}_u| \cdot |\mathcal{N}_v|}} \cdot h_v^{(k-1)},
	\end{equation}
	where $|\mathcal{N}_u|$ and $|\mathcal{N}_v|$ represent the number of neighbors of nodes $u$ and $v$, respectively. Therefore, term $\frac{1}{\sqrt{|\mathcal{N}_u| \cdot |\mathcal{N}_v|}}$ can be considered as the weight or contribution of each neighbor $v$ to the future representations of the central node $u$.
	
	The learning objective of the node classification task is to minimize the training loss (e.g., entropy loss), which is as follows.
	\begin{equation}\label{eq:loss_train}
		\min \limits_{\theta} \mathcal{L}_{\rm train} = - \sum_{v \in V_{\rm train}} \ln Z_{v,c}, \quad Z = f_\theta(G),
	\end{equation}
	where $Z \in \mathbb{R}^{n \times C}$ is the output of the last layer of GNN model $f_{\theta}$, $Z_{v,c}$ indicates the probability that node $v$ belongs to the true class $c$.

	\subsection{Graph Adversarial Attacks}
	As we mentioned before, the goal of node classification is to obtain the optimal model $f_{\theta^*}$ by minimizing $\mathcal{L}_{\rm train}$. Graph adversarial attacks, on the contrary, aim to fool a GNN model $f_{\theta}$ by generating a new adversarial graph $G^{\prime}$ via injecting some small perturbations into the original graph $G$ under a given budget $\Delta$, where the budget $\Delta$ represents the number of links can be flipped in structural attacks. More specifically, the goal of graph adversarial attacks can be formulated as follows.
	\begin{equation}\label{eq:loss_attack}
		\begin{aligned}
			& 	\min \mathcal{L}_{\rm atk}(f_{\theta^{\prime}}(G^{\prime}))
			\\
			& s.t. |G^{\prime} - G| \leq \Delta,
		\end{aligned}
	\end{equation}
	where $f_{\theta^{\prime}}$ represents the GNN model being attacked, which will be either trained on the clean graph $G$ in the evasion attacks (i.e., this work) or trained on the perturbed graph $G^{\prime}$ in the poisoning attacks. $\mathcal{L}_{\rm atk}$ represents the attack loss of the adversaries. As we focus on the targeted evasion attacks, we define the loss of the target node $u$ as the classification margin (${\rm CM}$) between the predicted probabilities after and before the attacks, which can be given as follows.
	\begin{equation}\label{eq:loss_attack1}
		\mathcal{L}_{\rm atk}(u) = f_{\theta^{\prime}}(G^{\prime})_{u,c} - f_{\theta^{\prime}}(G)_{u,c}.
	\end{equation}
	We try to minimize the ${\rm CM}$ between the probabilities that node $u$ be predicted to the ground truth label $c$ after and before the adversarial attacks under the same GNN model $f_{\theta^{\prime}}$. Obviously that ${\rm CM} \in [-1, 1]$, and a smaller ${\rm CM}$ indicates a better attack performance.

	\section{Proposed Methods}\label{sec:pm}
	In this section, we first introduce the definition of the noise of adversarial links, and then give the details of the proposed three attack methods based on noise propagation.
	
	\subsection{Noise Propagation}\label{sec:np}

	In the graph adversarial attacks, the key point to mislead the target node to a wrong class is to influence the neighborhood aggregations of GNNs. For example, the attackers can add some fake links between the target node and the nodes that may be harmful to the future representation of the target node. In addition to adding links, the attackers also can remove some original links between the target node and their neighbors. As previous studies \cite{bojchevski2019certifiable, jin2020adversarial} have shown that adding link attacks would be more powerful than removing link attacks, we only consider adding link attacks in this work for simplicity.
	
	Typically, we assume that our neighbors will help obtain a better representation via neighborhood aggregations. However, {\em what will happen if the information from our neighbors is harmful to the target node?} Recall to (\ref{eq:agg}), we can simply consider that the term $\frac{1}{\sqrt{|\mathcal{N}_u| \cdot |\mathcal{N}_v|}}$ ($Weight$ for short) as the aggregated weight of neighbor node $v$ to the central node $u$. Particularly, in the following, we define the potential node $v$ that chose to connect with the target node $u$ as the \textbf{adversarial node}, and the corresponding link $(u, v)$ as the \textbf{adversarial link}. As we only consider adding link attacks, the corresponding adversarial link $(u, v)$ should not exist in the clean graph. Moreover, we focus on the noise propagation on homophilic networks where most of the original neighbors of each node belong to the same class because traditional GNNs, such as GCN, usually cannot perform well on heterophilic networks, and corresponding heterophilic GNNs usually employ some new aggregation designs. Thus, the findings in this work maybe cannot be directly extended to the heterophilic networks, and we leave it for future work.	
	
	In the targeted attacks, for a specific target node $u$, if the aggregation between the target node $u$ and each possible adversarial node $v$ only contains noise and the specific noise value of each adversarial link is the same, we can have the following proposition based on (\ref{eq:agg}).

	\begin{proposition} \label{pro1}
		Let $G = (A, X, E)$ be a simple graph, and $Y = \{0,1,\cdots,C-1\}$ be the possible label. We simplify the feature of each node to be a one-hot vector corresponding to its label, denoted as $\mu(Y)$. Namely, the feature vector of node $u$ is $x_{u} = \mu(Y_u)$. Assuming that most of the original neighbors of each node belong to the same class, and the specific noise value of each adversarial link is the same. Consider a one-layer GCN where the output of node $u$ is $ h_u = \sigma(W \cdot \sum_{v \in \mathcal{N}(u)} \frac{1}{\sqrt{|\mathcal{N}_u| \cdot |\mathcal{N}_v|}} \cdot x_v)$, $\sigma$ is the softmax activation function, we have the following. 
		\begin{enumerate}
			\item From the perspective of target nodes, nodes with a lower degree will be easier to be attacked than those with a higher degree.
			\item From the perspective of adversarial nodes, nodes with a lower degree will influence the representation of the target node more than those with a higher degree.
		\end{enumerate} 
	\end{proposition}
    
	\begin{proof}
        See Appendix~\ref{proof:4.1} for the detailed proof.
        \end{proof}
	  we can further explain Proposition \ref{pro1} as follows. Intuitively, a target node with a higher degree will have a larger value of $|\mathcal{N}_u|$, thus leading to the fact that each neighbor will only contribute a small part in the neighborhood aggregation procedure as each of them has a small value of $Weight$. Therefore, if we only inject a small budget of adversarial links, it is unlikely to attack the node (i.e., influence the node representations) having a high degree successfully.
	
	From another perspective, for attacking a specific target node $u$, $|\mathcal{N}_u|$ is the same for all adversarial links, and thus a lower $|\mathcal{N}_v|$ indicates a much larger $Weight$. Therefore, these kinds of neighbors will contribute more to the target node in the aggregation procedure, indicating that nodes with a lower degree will be more aggressive than the nodes with a higher degree under this assumption (i.e., each adversarial link has the same noise value).

	However, the above observations are obtained based on the assumption that each adversarial link has the same noise value, but obviously the information from neighbors is more complicated, each link will have a different noise value on the propagation. Therefore, we have the further proposition as follows.

	\begin{proposition} \label{pro2}
		Except for the specific noise value of adversarial links varying from each other, we let all of the other assumptions be the same as Proposition \ref{pro1}. Then, we have the following. For a specific target node $u$, if the adversarial nodes have the same degree, the adversarial nodes which are dissimilar to node $u$ influence the aggregation of node $u$ more than those similar to node $u$.
	\end{proposition}

        \begin{proof}
        See Appendix~\ref{proof:4.2} for the detailed proof.
	\end{proof}
	Based on proposition \ref{pro2}, to better quantify the specific noise value of each adversarial link on the neighborhood aggregation process of target node $u$, we then propose a metric {\em link noise} (${\rm LN}$) considering both degree and similarity, which is defined as follows.
	\begin{equation}\label{eq:ln}
		{\rm LN}(u,v) = \frac{1}{\sqrt{(|\mathcal{N}_u|+1) \cdot (|\mathcal{N}_v|+1)}} \cdot {\rm DIS}(u,v), \; (u,v) \notin E,
	\end{equation}
	where the first term represents the latest $Weight$ of the adversarial node $v$ or the adversarial link ($u$, $v$). As we only focus on adding link attacks, the degree of both $u$ and $v$ will increase by one. The second term ${\rm DIS}(u,v)$ indicates the dissimilarity between the target node $u$ and corresponding adversarial node $v$, or in other words, the total noise between nodes $u$ and $v$ on aggregations. Based on (\ref{eq:agg}), we know only $\frac{1}{\sqrt{(|\mathcal{N}_u|+1) \cdot (|\mathcal{N}_v|+1)}}$ of the total noise will be aggregated and further influence the future representation of node $u$. Referring to previous work \cite{zhang2020cross,li2021deep}, we utilize the entropy of corresponding representations of nodes obtained from GNNs to characterize the noise or dissimilarity between the target node and the adversarial nodes, which is as follows.
	\begin{equation}\label{eq:dis}
		{\rm DIS}(u,v) = - \sum_{i = 0}^{C-1} (h_u^{(k)})_i \log (h_v^{(k)})_i, \quad (u,v) \notin E,
	\end{equation}
	where $(h_u^{(k)})_i$ is the $i$-th hidden representation of GNN model of node $u$ in the $k$-th layer. Actually, $(h_u^{(k)})_i$ is the probability that node $u$ belongs to class $i \in Y$ as if the total layer of the corresponding GNN is $k$. Particularly, since the attackers cannot exactly know the specific GNN model that is being attacked, we utilize the representation of classic GCN \cite{kipf2016semi} as the surrogate model in this work. From (\ref{eq:dis}), we can find that a higher similarity between the target node $u$ and adversarial node $v$ indicates a lower dissimilarity ${\rm DIS}(u,v)$.

	Combining (\ref{eq:ln}) and (\ref{eq:dis}), for a specific target node $u$, we can simplify the calculation by removing the same term as the target node $u$ is the same for all adversarial links. Therefore, we further define {\em appropriate link noise} ($\widetilde{\rm LN}$) as follows.
	\begin{equation}\label{eq:aln}
		\widetilde {\rm LN}(u,v) = \frac{- \sum_{i = 0}^{C-1} (h_u^{(k)})_i \log (h_v^{(k)})_i}{\sqrt{|\mathcal{N}_v|+1}}.
	\end{equation}
	
	Based on the above analysis, we can intuitively consider that a specific adversarial link $(u, v)$ with a higher $\widetilde {\rm LN}$ will influence the neighborhood aggregations of target node $u$ more than those with smaller $\widetilde {\rm LN}$. Therefore, it can be regarded as an importance measure of adversarial links. In the following, we further propose three attack strategies considering the proposed {\em appropriate link noise} metric.

	\subsection{Noise-Based Greedy Attacks}
	As analyzed above, $\widetilde{\rm LN}(u,v)$ indicates the noise value of potential adversarial links connecting $u$ and $v$. Therefore, we first propose a simple noise-based greedy attack method (NGA) by directly utilizing the noise value of each valid adversarial link. Specifically, we greedily add $\Delta$ valid adversarial links with the highest noise to the target node $u$. As the intuitive idea is that adversarial links with higher noise will negatively influence the neighborhood aggregation of the target node on a larger scale, we believe this greedy attack method will also achieve considerable performance.

	\begin{algorithm}[]
		\caption{Noise and Margin-based Attack (NMA)}
		\label{alg:nma}
		\raggedright\textbf{Input}: Original graph $G=(V,E,X)$, target node $u$, attack budget $\Delta$, size of candidates $\delta_1$.\\
		\textbf{Output}: Generated adversarial graph $G^{\prime} = (V, E^{\prime},X)$.\\
		
		\begin{algorithmic}[1] 
			\STATE Train a surrogate GCN model $f_{\theta}^s$ based on original $G$.
			\STATE $Z \gets $ Record the output of the last layer of model $f_{\theta}^s$ for all nodes.
			\STATE $\widetilde{\rm LN} \gets $ Calculate $\widetilde{\rm LN}(u,v)$ via $Z$ for all valid adversarial links $(u, v)$ based on (\ref{eq:aln}).
			\STATE $\widehat{\rm LN} \gets $ Sort the obtained $\widetilde{\rm LN}$ in descending order.
			\STATE $\hat{E} \gets $ Transform the sorted $\widehat{\rm LN}$ to the corresponding link list.
			\STATE $E_{\rm cand} \gets $ Construct the candidates by only retaining the top valid $\delta_1$ links with the highest noise, i.e., $\hat{E}[0:\delta_1-1]$.
			\STATE $E^{\prime} \gets E$.
			\FOR {each $i \in \{0, \cdots ,\Delta - 1\}$}
			\STATE Calculate $\mathcal{L}_{\rm atk}$ for all valid links in $E_{\rm cand}$ based on (\ref{eq:loss_attack1}), respectively.  
			\STATE $e \gets $ Select the valid link with the minimal $\mathcal{L}_{\rm atk}$.  
			\STATE $E^{\prime} = E^{\prime} \cup e$.
			\ENDFOR
			\RETURN $G' = (V, E^{\prime}, X)$.
		\end{algorithmic}
	\end{algorithm}

	\subsection{Noise and Margin-Based Attacks}
	Although the noise value of adversarial links can reflect the potential harmfulness of future aggregations to some extent, greedily adding the corresponding adversarial links based on the noise value may not yield the best attack performance. As NGA only utilizes the original noise value from the clean graph, it ignores the change of the latest noise value after the injection of some links. Therefore, combining the targeted attack loss in (\ref{eq:loss_attack1}), we propose a noise and classification margin-based attack (NMA) strategy, in which the major steps are similar to NETTACK \cite{zugner2018adversarial}. However, NETTACK needs to take all possible adversarial links into consideration, while the proposed NMA only retains a small ratio of links with the strong attack effect (i.e., higher noise value) to be the final candidates.

	\begin{table*}[]
		\centering
		\caption{The top 10 attack sequences ordered by CM of a target node (i.e. id = 422) on Cora dataset where the attack budgets are from 1 to 3. We simplify the adversarial links as the combination of the adversarial nodes as they all have the same target node. The boldfaced results indicate that the corresponding sequence exists in the top 10 sequences on the prior budget.}
			\begin{tabular}{c|cc|cc|cc}
				\bottomrule\bottomrule
				\multirow{2}{*}{\textbf{Index}} & \multicolumn{2}{c|}{\textbf{$\Delta$ = 1}}            & \multicolumn{2}{c|}{\textbf{$\Delta$ = 2}}                 & \multicolumn{2}{c}{\textbf{$\Delta$ = 3}}                      \\ \cline{2-7} 
				& \multicolumn{1}{c|}{\textbf{Sequence}}  & ${\textbf{CM}}$     & \multicolumn{1}{c|}{\textbf{Sequence}}       & ${\textbf{CM}}$     & \multicolumn{1}{c|}{\textbf{Sequence}}            & ${\textbf{CM}}$     \\ \hline\hline
				1                      & \multicolumn{1}{c|}{1669} & -0.3245 & \multicolumn{1}{c|}{\textbf{1669}, 2167} & -0.4459 & \multicolumn{1}{c|}{\textbf{1669}, \textbf{2167}, 2259} & -0.4742 \\ \hline
				2                      & \multicolumn{1}{c|}{2167} & -0.3202 & \multicolumn{1}{c|}{\textbf{1669}, 2259} & -0.4452 & \multicolumn{1}{c|}{\textbf{1669}, 2168, \textbf{2259}} & -0.4739 \\ \hline
				3                      & \multicolumn{1}{c|}{2259} & -0.3181 & \multicolumn{1}{c|}{\textbf{1669}, 2168} & -0.4443 & \multicolumn{1}{c|}{1281, \textbf{1669}, \textbf{2167}} & -0.4733 \\ \hline
				4                      & \multicolumn{1}{c|}{2168} & -0.3135 & \multicolumn{1}{c|}{\textbf{2167}, 2259} & -0.4439 & \multicolumn{1}{c|}{1281, \textbf{1669}, \textbf{2259}} & -0.4733 \\ \hline
				5                      & \multicolumn{1}{c|}{1281} & -0.3057 & \multicolumn{1}{c|}{2168, \textbf{2259}} & -0.4424 & \multicolumn{1}{c|}{1281, \textbf{1669}, \textbf{2168}} & -0.4731 \\ \hline
				6                      & \multicolumn{1}{c|}{234}  & -0.2908 & \multicolumn{1}{c|}{1281, \textbf{1669}} & -0.4417 & \multicolumn{1}{c|}{\textbf{1669}, \textbf{2167}, 2168} & -0.4730 \\ \hline
				7                      & \multicolumn{1}{c|}{697}  & -0.2899 & \multicolumn{1}{c|}{1281, \textbf{2167}} & -0.4401 & \multicolumn{1}{c|}{1281, \textbf{2167}, \textbf{2259}} & -0.4728 \\ \hline
				8                      & \multicolumn{1}{c|}{535}  & -0.2563 & \multicolumn{1}{c|}{1281, \textbf{2259}} & -0.4397 & \multicolumn{1}{c|}{1281, \textbf{2168}, \textbf{2259}} & -0.4725 \\ \hline
				9                      & \multicolumn{1}{c|}{342}  & -0.2548 & \multicolumn{1}{c|}{1281, \textbf{2168}} & -0.4388 & \multicolumn{1}{c|}{\textbf{1669}, \textbf{2167}, 234}  & -0.4725 \\ \hline
				10 & \multicolumn{1}{c|}{2404} & -0.2527 & \multicolumn{1}{c|}{\textbf{2167}, 2168} & -0.4383 & \multicolumn{1}{c|}{\textbf{2167}, 2168, \textbf{2259}} & -0.4724 \\ \bottomrule\bottomrule
			\end{tabular}%
		\label{table:cm}
	\end{table*}

	Specifically, we first utilize the ranking of  $\widetilde{\rm LN}$ to control the size of final candidates, as the links with higher noise tend to mislead the target node more in the neighborhood aggregations. The detailed procedure of NMA is given in Algorithm \ref{alg:nma}. We only reserve the top $\delta_1$ valid links with higher noise to be the final candidates before we move into the margin calculation step (i.e., {\em line 6}). By employing the above candidate refining mechanism, the searching space is greatly reduced as we only need to evaluate a small part of link perturbations.

	\begin{algorithm}[t]
		\caption{Noise and Margin-based Attack-Boost (NMAB)}
		\label{alg:nmab}
		\raggedright\textbf{Input}: Original graph $G=(V,E,X)$, target node $u$, attack budget $\Delta$, size of candidates $\delta_2$, size of single optimal list $len_{\rm sin}$, size of retain list $len_{\rm re}$.\\
		\textbf{Output}: Generated adversarial graph $G^{\prime} = (V, E^{\prime},X)$.\\
		
		\begin{algorithmic}[1] 
			\STATE Train a surrogate GCN model $f_{\theta}^s$ based on original $G$.
			\STATE $Z \gets $ Record the output of the last layer of model $f_{\theta}^s$ for all nodes.
			\STATE $\widetilde{\rm LN} \gets $ Calculate $\widetilde{\rm LN}(u,v)$ via $Z$ for all valid adversarial links $(u, v)$ based on (\ref{eq:aln}).
			\STATE $\widehat{\rm LN} \gets $ Sort the obtained $\widetilde{\rm LN}$ in descending order.
			\STATE $\hat{E} \gets $ Transform the sorted $\widehat{\rm LN}$ to the corresponding link list.
			\STATE $E_{\rm cand} \gets $ Construct the candidates by only retaining the top $\delta_2$ valid links with the highest noise, i.e., $\hat{E}[0:\delta_2-1]$.

			\STATE Calculate $\mathcal{L}_{\rm atk}$ for all links in $E_{\rm cand}$ based on (\ref{eq:loss_attack1}), respectively.  
			\STATE $E_{\rm sin} \gets $ Record the top $len_{\rm sin}$ links which have the highest $\mathcal{L}_{\rm atk}$ among all $\delta_2$ candidate links.
			
			\STATE $E_{\rm optimal} \gets $Record the top $len_{\rm re}$ links which have the highest $\mathcal{L}_{\rm atk}$ among all $\delta_2$ candidate links.
			
			\FOR {each $i \in \{1, \cdots ,\Delta - 1\}$}
			
			\STATE $E_{\rm current} \gets$ Construct the valid ($i+1$)-length attack list by combining the optimal $(1)$-length attack list $E_{\rm sin}$ and the optimal ($i$)-length attack list $E_{\rm optimal}$.  
			
			\STATE Calculate $\mathcal{L}_{\rm atk}$ for all link sequences in $E_{\rm current}$ based on (\ref{eq:loss_attack1}), respectively.
			\STATE $E_{\rm optimal} \gets $ Select the top $len_{\rm re}$ attack sequence with the minimal $\mathcal{L}_{\rm atk}$.
			\ENDFOR
			\STATE $E^{*} \gets $ Select the sequence with the minimal $\mathcal{L}_{\rm atk}$ in $E_{\rm optimal}$ as the optimal attack sequence.
			\STATE $E^{\prime} = E \cup E^{*}$.
			\RETURN $G' = (V, E^{\prime}, X)$.
		\end{algorithmic}
	\end{algorithm}

	\subsection{Noise and Margin-Based Attacks-Boost}
	In the above, we propose NMA by taking the noise value of links as a candidate selection metric, and further utilising the classification margin as the final indicator. But similar to NETTACK, our NMA also focuses on single step optimization as we will select the current optimal adversarial link decreasing the classification margin the most during each attack budget. As a result, we may ignore other adversarial link combinations which have more powerful attack performance in a multiple steps optimization perspective. Based on the traditional robustness evaluation study of complex networked systems \cite{zhu2014revealing}, an interesting finding is that the powerful ($i+1$)-length attack sequences usually contain the strong ($i$)-length attack sequences. It is also straightforward that the combination of strong sequences will lead to a better attack performance.

	To investigate whether this phenomenon exists in our problem, we conduct a similar empirical study on a specific target node (i.e., id = 422) of the Cora dataset, whose statistics will be given in Section \ref{sec:data}. Particularly, to reduce the full searching space, we only consider the adversarial links in NMA that are controlled by the size of candidates $\delta_1$ as the valid links for simplicity. By utilizing the classification margin as the evaluation metric, we observe that this phenomenon also exists in our task, as shown in Table \ref{table:cm}. For example, the optimal (2)-link attack sequence (i.e., \{1669, 2167\}) contains the optimal (1)-link attack sequence (i.e., \{1669\}), the optimal (3)-link attack sequence (i.e., \{1669, 2167, 2259\}) contains the optimal (2)-link attack attack sequence (i.e., \{1669, 2167\}), etc. Therefore, we know that the strong attack sequences in prior steps can help search for more powerful attack sequences in the next step, which can be considered a multiple steps optimization process.

	Therefore, following the above findings, we propose an improved strategy, NMAB, to boost the attack performance of NMA. The major steps of NMAB are similar to NMA, as shown in Algorithm \ref{alg:nmab}. The only difference is, NMAB will not greedily select the adversarial link with the lowest CM during each step, but maintain two optimal lists with pre-defined sizes including the optimal ($1$)-link attack sequence and optimal ($i$)-link attack sequence (i.e., {\em lines 8-14}). Specifically, $E_{\rm sin}$ represents the optimal (1)-link attack sequence with the size of $len_{\rm sin}$, and $E_{\rm optimal}$ indicates the optimal ($i$)-link attack sequence with the size of $len_{\rm re}$. For each attack budget, we will construct the current attack candidates (i.e., sequence length $ = i+1$) by combining all the valid combinations of list $E_{\rm optimal}$ (i.e., sequence length $= i$) and list $E_{\rm sin}$ (i.e., sequence length $= 1$). Finally, we will return the optimal attack sequence in the $(\Delta)$-length attack list (i.e., $E^{*}$) to generate the optimal adversarial links. It is worth noting that, compared to NMA which focuses on selecting the current optimal link during each attack budget from a local perspective, NMAB can select the optimal link(s) in multiple attack budgets from a more global perspective.

    \begin{figure*}[t]
        \centering
        \includegraphics[width=1\textwidth]{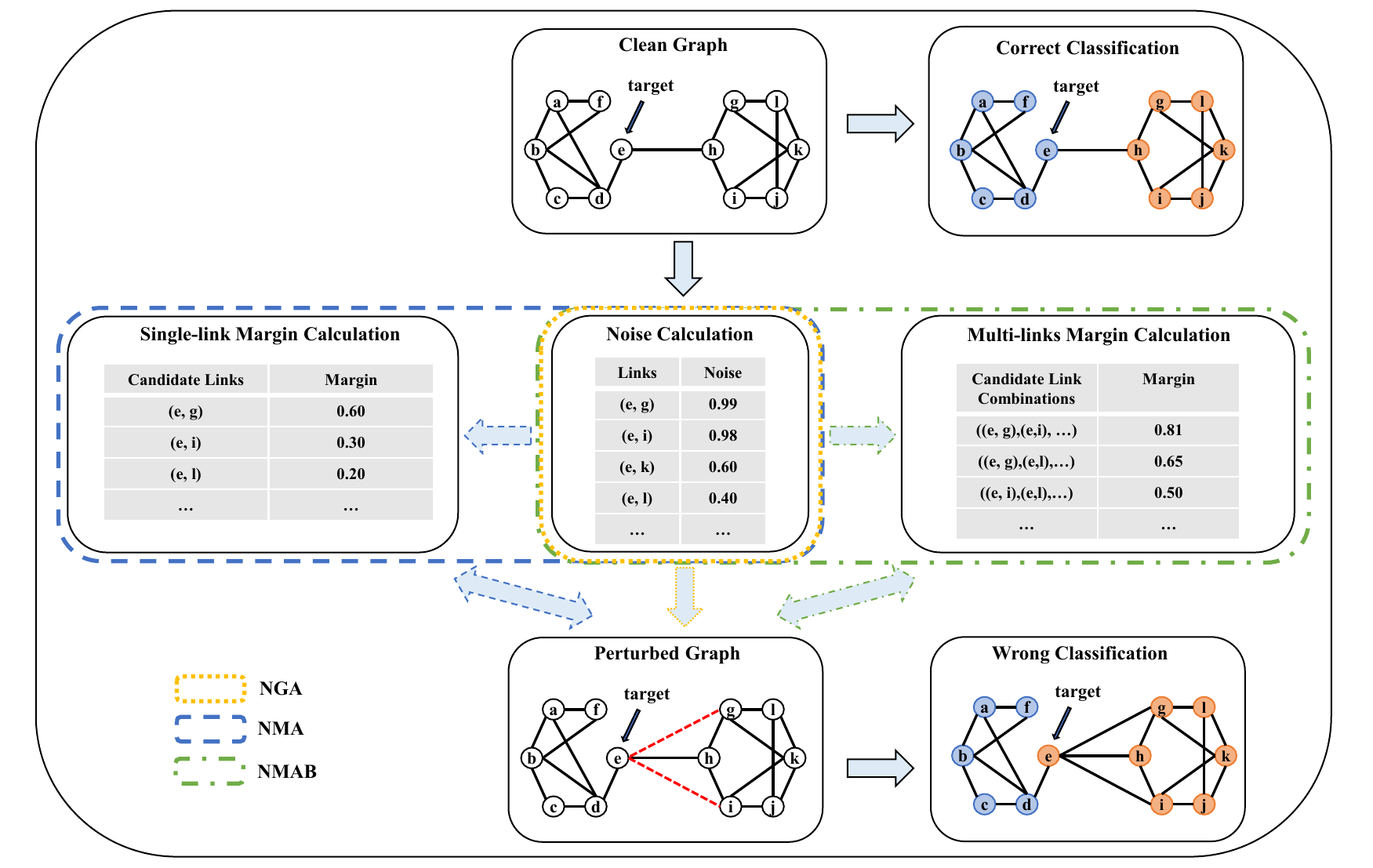} 
        \caption{A systematic framework of targeted attacks based on the proposed three strategies.}
        \label{fig:framework}
    \end{figure*}

	\subsection{Overall Framework}
	To sum up, based on the proposed noise of nodes, we propose three attack strategies including NGA, NMA, and NMAB from different perspectives. The overall framework of the attack procedure based on the proposed three strategies is given in Fig. \ref{fig:framework}. In the targeted attack scenario, we will first select a node among all possible nodes as the target node, and then employ the specific strategies (i.e., NGA, NMA, or NMAB) to generate the optimal adversarial links. Finally, we will evaluate the attack performance via the GNN model trained on the clean graph by checking whether the target node will be predicted to a wrong class.

	\section{Experiments}\label{sec:e}
	In this section, we conduct various experiments to verify the effect of the proposed three attack strategies. To be specific, we first introduce the statistical details of the datasets. Then, we demonstrate the baselines for comparisons and the GNNs being attacked. Following that, we introduce the specific parameter settings of models and methods. Finally, we present the experimental results together with the corresponding analysis. The code for reproduction will be publicly available at \href{https://github.com/alexfanjn/Noise-propagation-attack}{https://github.com/alexfanjn/Noise-propagation-attack}, depending on the acceptance.
	
	\subsection{Datasets}\label{sec:data}
	We evaluated the proposed strategies on three representative benchmark datasets including Cora, Citeseer, and Pubmed \cite{sen2008collective}. Following the settings in the previous study \cite{dai2018adversarial}, we extracted the largest connected component of them, and the statistical information is given in Table \ref{tab:statistic}.

	\subsection{Baselines}
	We chose three recent methods as baselines to compare with the proposed NGA, NMA, and NMAB. The details are as follows.
	\begin{enumerate}
		\item \textbf{NETTACK} \cite{zugner2018adversarial}: NETTACK is a strong baseline that attacks the structures and features based on classification margin. Since we focus on structure attacks, NETTACK is only allowed to attack the topological connections for fair comparison.
		
		\item \textbf{FGA} \cite{chen2018fast}: FGA is a gradient-based targeted attack method. Based on the assumption that the adversarial links with the maximum absolute gradient will influence the loss function the most, FGA greedily chooses the corresponding links.

		\item \textbf{SGA} \cite{li2021adversarial}: SGA is another gradient-based targeted attack method by only considering the $k$-hop subgraph rather than the entire graph. Specifically, SGA largely reduces the size of the candidate set because the adding link operations will only occur when the nodes belong to the second possible class of the target node.
	\end{enumerate}
	Among them, NETTACK and SGA used SGC as the surrogate model, while FGA and our methods adopted GCN as the surrogate model. Particularly, we do not compared our methods with other attack strategies such as GradArgmax \cite{dai2018adversarial}, DICE \cite{waniek2018hiding}, etc.,  as FGA and SGA already achieved a better attack performance than them.

	\begin{table*}[t]
		\centering
		\caption{Attack success rate of the baselines against GNNs on three datasets. The last column indicates the average rank of each baseline among different models and datasets. The top 3 results in each column and the best rank are boldfaced.}
        \resizebox{\linewidth}{!}{
			\begin{tabular}{c|ccc|ccc|ccc|c}
				\bottomrule\bottomrule
				\multirow{2}{*}{\textbf{Methods}} &
				\multicolumn{3}{c|}{\textbf{Cora}} &
				\multicolumn{3}{c|}{\textbf{Citeseer}} &
				\multicolumn{3}{c|}{\textbf{Pubmed}} &
				\multicolumn{1}{c}{\multirow{2}{*}{\textbf{Ranks}}} \\ \cline{2-10}
				&
				\multicolumn{1}{c|}{\textbf{GCN}} &
				\multicolumn{1}{c|}{\textbf{SGC}} &
				\textbf{GAT} &
				\multicolumn{1}{c|}{\textbf{GCN}} &
				\multicolumn{1}{c|}{\textbf{SGC}} &
				\textbf{GAT} &
				\multicolumn{1}{c|}{\textbf{GCN}} &
				\multicolumn{1}{c|}{\textbf{SGC}} &
				\textbf{GAT} &
				\multicolumn{1}{c}{} \\ \hline\hline
				NETTACK &
				\multicolumn{1}{c|}{\textbf{0.9044}} &
				\multicolumn{1}{c|}{\textbf{0.9278}} & \textbf{0.7388}
				& 
				\multicolumn{1}{c|}{0.8344} &
				\multicolumn{1}{c|}{\textbf{0.8614}} & \textbf{0.6496}
				&
				\multicolumn{1}{c|}{0.9684} &
				\multicolumn{1}{c|}{\textbf{0.9874}} & \textbf{0.9114}
				& 2.22
				\\ \hline
				FGA &
				\multicolumn{1}{c|}{0.9030} &
				\multicolumn{1}{c|}{0.8184} & 0.5472
				&
				\multicolumn{1}{c|}{\textbf{0.8842}} &
				\multicolumn{1}{c|}{0.8526} &0.5156
				&
				\multicolumn{1}{c|}{\textbf{0.9772}} &
				\multicolumn{1}{c|}{\textbf{0.9630}} &\textbf{0.9014}
				&3.56
				\\ \hline
				SGA &
				\multicolumn{1}{c|}{0.8540} &
				\multicolumn{1}{c|}{0.8242} &0.5866
				&
				\multicolumn{1}{c|}{0.6686} &
				\multicolumn{1}{c|}{0.6722} &0.4708
				&
				\multicolumn{1}{c|}{0.9604} &
				\multicolumn{1}{c|}{0.9546} &0.8510
				&4.89
				\\ \hline
				NGA (Ours) &
				\multicolumn{1}{c|}{0.7468} &
				\multicolumn{1}{c|}{0.6746} &0.4626
				&
				\multicolumn{1}{c|}{0.6222} &
				\multicolumn{1}{c|}{0.6854} &0.6022
				&
				\multicolumn{1}{c|}{0.9072} &
				\multicolumn{1}{c|}{0.8746} &0.7426
				&5.67
				\\ \hline
				NMA (Ours) &
				\multicolumn{1}{c|}{\textbf{0.9360}} &
				\multicolumn{1}{c|}{\textbf{0.8738}} &\textbf{0.6082}
				&
				\multicolumn{1}{c|}{\textbf{0.9342}} &
				\multicolumn{1}{c|}{\textbf{0.8782}} &\textbf{0.6368}
				&
				\multicolumn{1}{c|}{\textbf{0.9802}} &
				\multicolumn{1}{c|}{0.9536} &0.8872
				&3.00
				\\ \hline
				NMAB (Ours) &
				\multicolumn{1}{c|}{\textbf{0.9676}} &
				\multicolumn{1}{c|}{\textbf{0.9166}} &\textbf{0.6580}
				&
				\multicolumn{1}{c|}{\textbf{0.9532}} &
				\multicolumn{1}{c|}{\textbf{0.9074}} &\textbf{0.6702}
				&
				\multicolumn{1}{c|}{\textbf{0.9842}} &
				\multicolumn{1}{c|}{\textbf{0.9598}} &\textbf{0.8954}
				&\textbf{1.67}
				\\ \bottomrule\bottomrule
		\end{tabular}
        }
	\label{table:asr}
	\end{table*}
	
	\subsection{Targeted Models}
	For the GNN models to be attacked, we selected three traditional GNNs including GCN, SGC, and GAT as representatives. The details of these methods are as follows.
	
	\begin{table}[]
		\centering
		\caption{The statistics of datasets.}
			\begin{tabular}{c|c|c|c|c|c}
				\bottomrule\bottomrule
				\textbf{Datasets} &
				\textbf{\#Nodes} &
				\textbf{\#Links} &
				\textbf{\#Features} &
				\textbf{\#Classes} &
				\textbf{Avg. Degree} \\ \hline\hline
				Cora     & 2,485  & 5,069  & 1,433 & 7 & 4.08 \\ \hline
				Citeseer & 2,100  & 3,668  & 3,703 & 6 & 3.48 \\ \hline
				Pubmed   & 19,717 & 44,324 & 500   & 3 & 4.50 \\ \bottomrule\bottomrule
		\end{tabular}

		\label{tab:statistic}
		
	\end{table}
	
	\begin{enumerate}
		\item \textbf{GCN} \cite{kipf2016semi}: GCN is a traditional GNN that obtains the low dimensional representation of nodes via aggregating the local structural and feature information of neighbors.
		
		\item \textbf{SGC} \cite{wu2019simplifying}: SGC is the simplified version of GCN by removing the activation functions in the middle layers of GCN, which achieves comparable performance as GCN but requires smaller training complexity.

		\item \textbf{GAT} \cite{velivckovic2017graph}: GAT further improves the performance of the original GCN by proposing a masked self-attention mechanism. In this way, GAT gives different weights to different neighbors, rather than directly assigning the weight of each neighbor based on its degree information.
		
	\end{enumerate}
	Particularly, as the proposed attack strategies were trained based on the surrogate GCN model in this work, the attack performance on other GNNs can show the generalization ability of our methods.

	\begin{table*}[t]
		\centering
		\caption{Attack success rate of the proposed methods against GCN under different dissimilarity metrics on three datasets, where EUC, COS, and ENT refer to Euclidean distance, Cosine distance, and entropy distance, respectively. The best results in each method under the same dataset are boldfaced.}
			\begin{tabular}{c|ccc|ccc|ccc}
				\bottomrule\bottomrule
				\multirow{2}{*}{\textbf{Methods}} &
				\multicolumn{3}{c|}{\textbf{Cora}} &
				\multicolumn{3}{c|}{\textbf{Citeseer}} &
				\multicolumn{3}{c}{\textbf{Pubmed}} \\ \cline{2-10}
				&
				\multicolumn{1}{c|}{\textbf{EUC}} &
				\multicolumn{1}{c|}{\textbf{COS}} &
				\textbf{ENT} &
				\multicolumn{1}{c|}{\textbf{EUC}} &
				\multicolumn{1}{c|}{\textbf{COS}} &
				\textbf{ENT} &
				\multicolumn{1}{c|}{\textbf{EUC}} &
				\multicolumn{1}{c|}{\textbf{COS}} &
				\textbf{ENT} \\ \hline\hline
				NGA &
				\multicolumn{1}{c|}{0.679} &
				\multicolumn{1}{c|}{0.778} & \textbf{0.781}
				& 
				\multicolumn{1}{c|}{0.451} &
				\multicolumn{1}{c|}{\textbf{0.666}} & 0.647
				&
				\multicolumn{1}{c|}{0.721} &
				\multicolumn{1}{c|}{0.749} & \textbf{0.879}
				
				\\ \hline
				NMA &
				\multicolumn{1}{c|}{0.865} &
				\multicolumn{1}{c|}{0.869} & \textbf{0.948}
				&
				\multicolumn{1}{c|}{0.807} &
				\multicolumn{1}{c|}{0.821} &\textbf{0.914}
				&
				\multicolumn{1}{c|}{0.894} &
				\multicolumn{1}{c|}{0.901} &\textbf{0.976}
				
				\\ \hline
				NMAB &
				\multicolumn{1}{c|}{0.946} &
				\multicolumn{1}{c|}{0.906} &\textbf{0.975}
				&
				\multicolumn{1}{c|}{0.917} &
				\multicolumn{1}{c|}{0.868} &\textbf{0.936}
				&
				\multicolumn{1}{c|}{0.931} &
				\multicolumn{1}{c|}{0.913} &\textbf{0.978}
				
				\\ \bottomrule\bottomrule
		\end{tabular}
		\label{table:metrics}
		
	\end{table*}

	\subsection{Parameter Settings}
	
	Each datasets is randomly split into the training set (10\%), validation set (10\%), and test set (80\%). All of the experimental results are the average performance among 5 different splits. Moreover, all of the baselines and attacked GNNs are implemented by utilizing the opensource platform {\em DeepRobust} \cite{jin2020adversarial}.

	For the specific attacks, we randomly selected 1000 nodes in the test set as the target nodes for each dataset. To ensure the unnoticeable perturbations, we followed the common configuration of several studies \cite{zugner2018adversarial,li2021adversarial} by setting the attack budget $\Delta$ as the degree of target node for all attack methods. For NMA, we set the size of candidate $\delta_1$ as $5\Delta$. For NMAB, we set the size of candidate $\delta_2$, the size of single optimal list $len_{\rm sin}$,  and size of retain list ${len_{\rm re}}$ as $10\Delta$, $3\Delta$, and 10, respectively. Except for the above settings, all other parameter settings of the baselines and attacked models are adopted to the default settings.

	\begin{figure*}[t]
		\centering
		\subfigure{
			\begin{minipage}[]{0.33\linewidth}
				\includegraphics[scale=0.28]{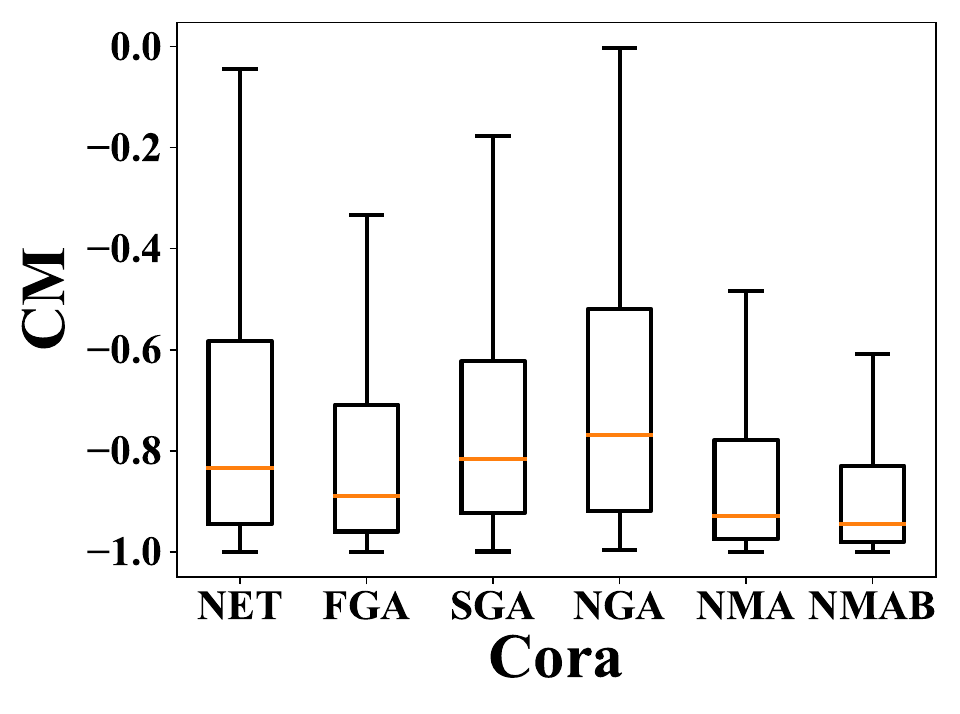}
			\end{minipage}%
		}%
		\subfigure{
			\begin{minipage}[]{0.33\linewidth}
				\includegraphics[scale=0.28]{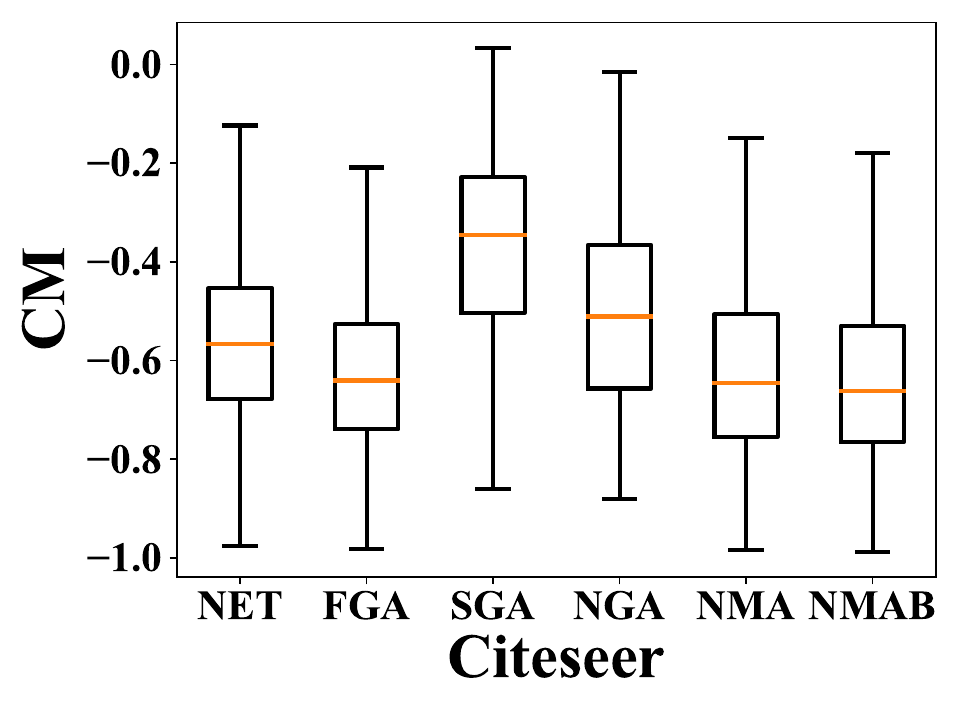}
			\end{minipage}%
		}%
		\subfigure{
			\begin{minipage}[]{0.33\linewidth}
				\includegraphics[scale=0.28]{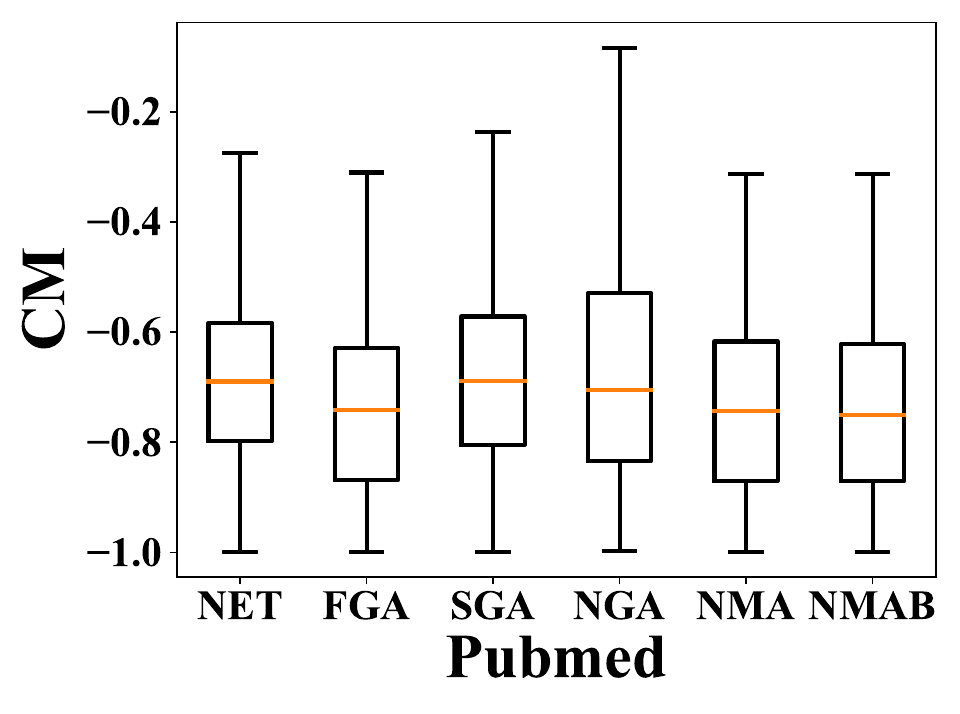}
			\end{minipage}%
		}%
		
		\caption{Classification margin comparisons between different methods on three datasets.}
		\label{fig:cm}
	\end{figure*}

	\subsection{Experimental Results and Analysis}
	\subsubsection{Attack Success Rate}
	We first investigate the attack success rate of different methods under different datasets and models. Attack success rate refers to the ratio between the number of target nodes that are attacked successfully and the total number of target nodes. From Table \ref{table:asr}, we can find that NGA obtains a good attack performance in some cases even though its idea is simple. For the proposed three strategies, NMAB performs the best, NMA is better than NETTACK in most cases, and both NMA and NMAB are better than NGA to a large extent. We think the slight gaps between NMA and NETTACK in some cases are acceptable as NMA has a much smaller size of candidates than NETTACK, so NETTACK will reasonably have more chances to obtain better solutions. For the comparisons with other baselines based on the overall ranks (i.e., the last column of Table \ref{table:asr}), NMAB obtains the overall best attack success rate over all other methods, indicating the superiority of its multiple steps optimization. Following NMAB, NMA obtains a comparable performance as NETTACK, since they rank third and second place, respectively. The above results suggest that, even though NETTACK has a larger candidate set than NMAB, NETTACK may drop into local optimums due to its single step optimization. Moreover, although we utilize GCN as the surrogate model, the generated adversarial links also achieve remarkable attack performance on both SGC and GAT models, showing the generalizability of the proposed methods.

	\subsubsection{Classification Margin}
	 We also analyze the classification margin (CM) obtained by each baseline. Fig. \ref{fig:cm} is the box plot of the corresponding classification margins of the GCN model on three datasets. We can find that NMA and NMAB are the best two strategies that obtain the optimal ${\rm CM}$ than all of the other baselines, and NMAB is slightly better than NMA. More importantly, the worst cases (i.e., highest CM) of NMA and NMAB are also better than other strategies mostly. The above findings demonstrate that the proposed strategies can obtain a comparable or even better attack performance than current methods even though the size of candidate adversarial links has been largely reduced via the rankings of noise values, especially for NMA and NMAB.

	\subsubsection{Impact of Different Dissimilarity Metrics} 
	In the default settings, we adopt entropy to measure the dissimilarity of different node pairs. To further investigate the impact of different dissimilarity metrics, we employ two other classic distance metrics including Euclidean distance and Cosine distance, to characterize the difference of the representation vectors of nodes. Specifically, we replace the calculation of entropy in (\ref{eq:dis}) to the specific calculation of Euclidean distance and Cosine distance during the candidate selection process in the proposed methods. Then we analyze the corresponding attack success rate under different dissimilarity measurements. As shown in Table \ref{table:metrics}, entropy distance obtains the overall best attack performance in the proposed three methods among all datasets. The above results indicate that, compared with traditional Euclidean distance and Cosine distance, entropy distance may be a better metric to characterize the dissimilarity of different nodes during the neighborhood aggregations of GNNs.

	\begin{table}[]
		\centering
		\caption{Class distributions of selected adversarial nodes on three datasets. `Sec. Poss.' refers to the second possible class.}
		\begin{tabular}{c|c|c|c|c}
			\bottomrule\bottomrule
			\textbf{Datasets} & \textbf{Methods} & \textbf{Same} & \textbf{Sec. Poss.} & \textbf{Others} \\ \hline\hline
			\multirow{2}{*}{Cora}     & NMA  & 0  & 0.3054 & 0.6946 \\ \cline{2-5} 
			& NMAB &   0     &   0.3130   &  0.6870     \\ \hline
			\multirow{2}{*}{Citeseer} & NMA  & 0.0100  & 0.3370 & 0.6619 \\ \cline{2-5} 
			& NMAB & 0.0016       & 0.3548     &  0.6436    \\ \hline
			\multirow{2}{*}{Pubmed}   & NMA  & 0 & 0.8936 & 0.1064 \\ \cline{2-5} 
			& NMAB &   0     &   0.8891   &  0.1109    \\ \bottomrule\bottomrule
		\end{tabular}%
		\label{tab:class}
	\end{table}

	\begin{figure}[t]
		\centering	
		\includegraphics[width=0.4\linewidth]{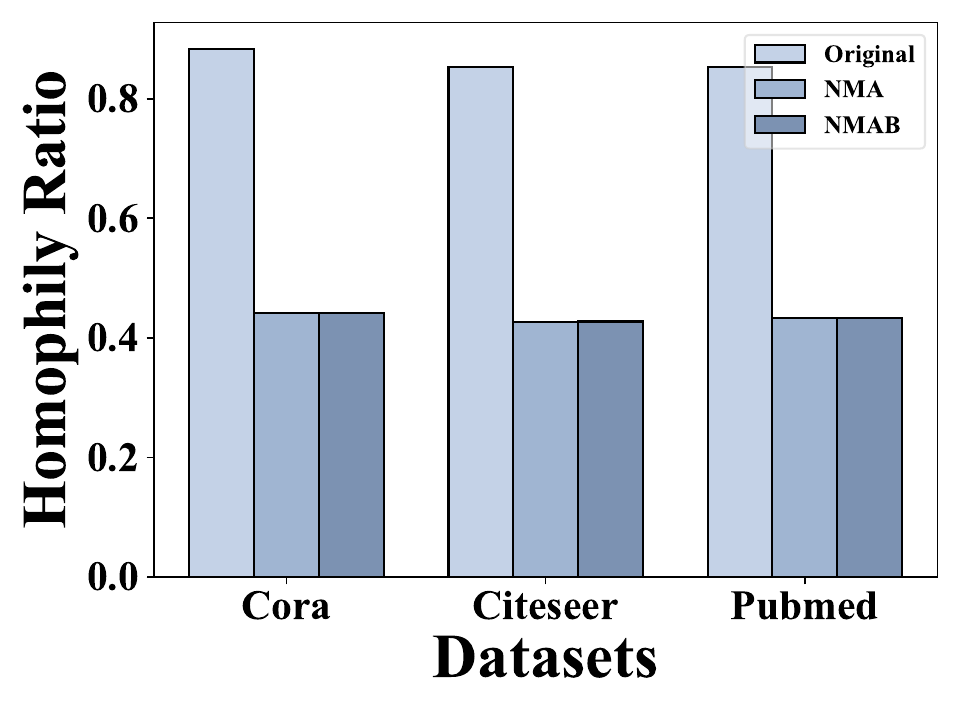} 
		\caption{Average homophily ratios of target nodes before and after NMA/NMAB attacks on three datasets.}
		\label{fig:homo}
	\end{figure}

	\subsubsection{Preference of Adversarial Node Selections}
	 In this subsection, we analyze the preference of adversarial node selections of the proposed methods. As NMA and NMAB perform better than NGA, we utilize the prior two as representative strategies. Specifically, the investigated properties include class distribution, degree, and predicted confidence of the selected nodes.

	\begin{figure*}[]
		\subfigure{
			\begin{minipage}[]{0.33\linewidth}
				\includegraphics[scale=0.28]{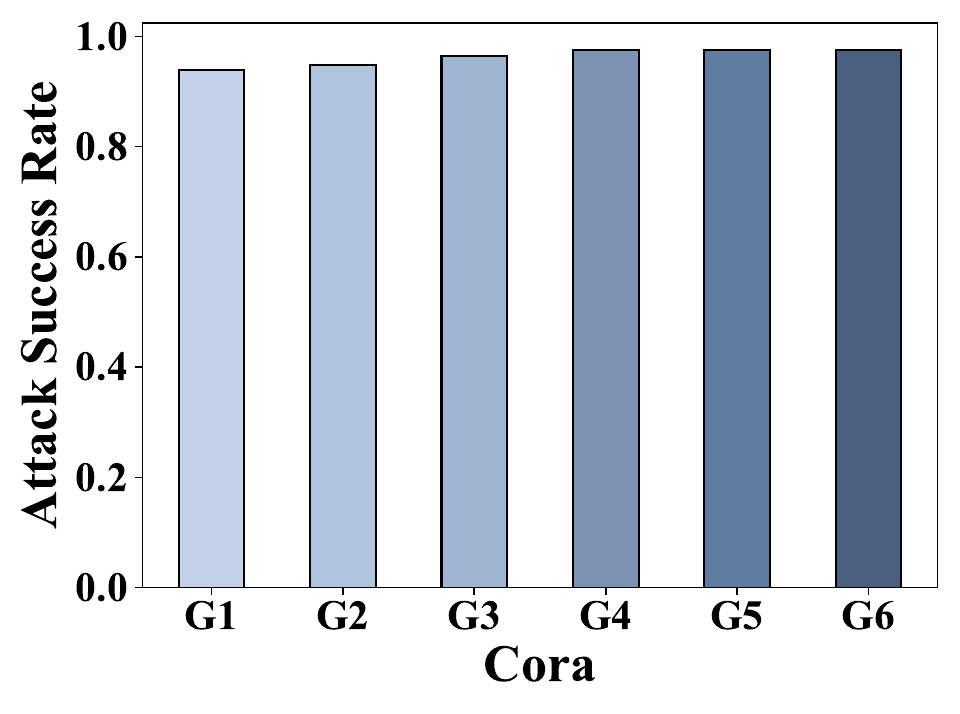}
			\end{minipage}%
		}%
		\subfigure{
			\begin{minipage}[]{0.33\linewidth}
				\includegraphics[scale=0.28]{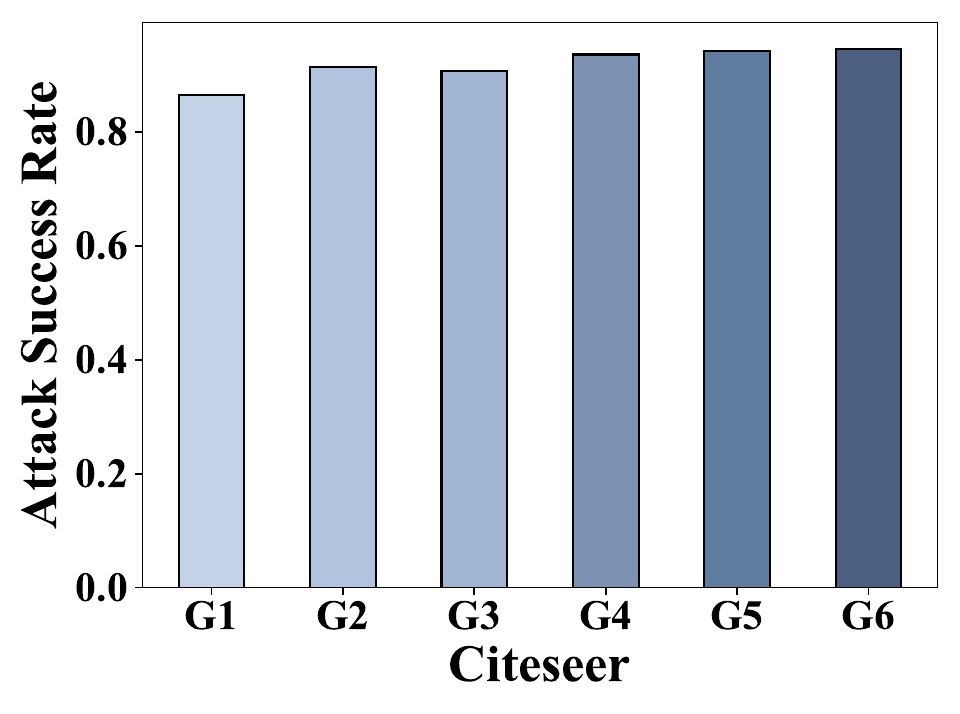}
			\end{minipage}%
		}%
		\subfigure{
			\begin{minipage}[]{0.33\linewidth}
				\includegraphics[scale=0.28]{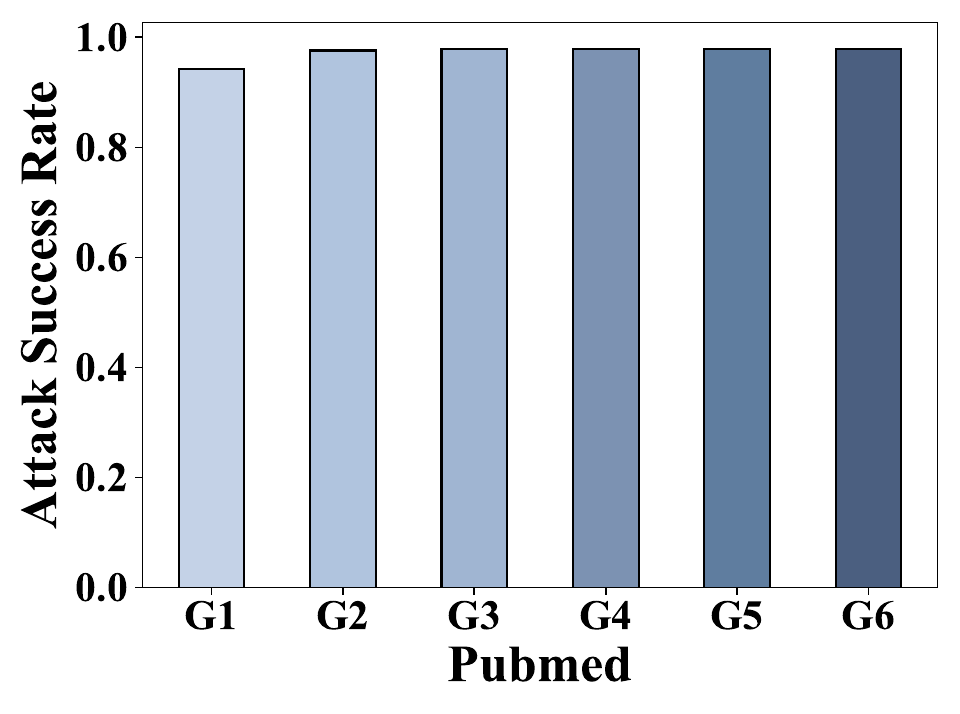}
			\end{minipage}%
		}%
		\centering
		\caption{Attack success rate of the proposed NMA and NMAB methods on three datasets under different parameter combinations.}
		\label{fig:pa}
	\end{figure*}

	\textbf{Class.}
	To explore the class of the selected adversarial nodes, we first analyze the change of homophily of target nodes before and after the attacks. Homophily \cite{pei2020geom, zhu2020beyond} is widely adopted to characterize the similarity of the target node and its neighbors. Specifically, the homophily of a specific node $i$ is given by the ratio of the number of neighbors with the same label to the number of all neighbors, which is as follows.
	
	\begin{equation}\label{eq:homo}
		{\rm homophily}_i = \frac{\# \ i's \ neighbors \ with \ same 	\ label}{\# \ i's \ all \ neighbors}. 
	\end{equation}
	
	Fig. \ref{fig:homo} illustrates the average homophily of target nodes before and after the attacks on three datasets, we can easily find that both NMA and NMAB tend to decrease the homophily of target nodes. In other words, both of them will likely connect the target node with those nodes whose classes are different from the target node. 
	
	To investigate the class distributions of the new neighbors, we divide different classes into three groups, namely the same class, the second possible class, and others. The first group refers to the case where the new neighbor belongs to the same class as the target node. The second group refers to the case where the new neighbors belong to the second possible class of the target node, which is consistent with the node selection idea of SGA, while the last group indicates other classes that do not meet the prior two groups. 
	
	As shown in Table \ref{tab:class}, almost none of the new neighbors are selected from the same class as the targeted node, and most of them belong to different classes (i.e., either the second possible class or others). The above finding is consistent with the intuitive fact that neighbors with the same class usually will help central nodes obtain a better representation, while the nodes with different classes usually will be harmful to the feature aggregation of central nodes. As for the other two groups, our methods do not always select nodes from the second possible class but give other classes large weights, except for the Pubmed dataset since it is only a three-class dataset. This phenomenon indicates the idea that SGA follows where the second possible class is the easiest way to achieve the wrong prediction may not always hold. Connecting the target node with the new neighbors belonging to other classes may lead to a better attack performance than connecting them with the nodes belonging to the second possible class.

	\textbf{Degree.} The first part of Table \ref{tab:dac} shows the degree properties of the selected adversarial nodes. As we can observe, the adversarial nodes usually have a low degree, which is even lower than the average degree of nodes in the clean datasets. The above finding supports our analysis in Section \ref{sec:np} that low degree nodes usually will have a larger influence on the neighborhood aggregation mechanism than high degree nodes.
	
	\textbf{Predicted Confidence.} Besides the above properties, we further analyze the average predicted confidence of the adversarial nodes. As shown in the second part of Table \ref{tab:dac}, Both NMA and NMAB tend to choose the nodes with relatively high confidence in the predicted class, which also refers to the fact that these kinds of low degree nodes usually have a higher value of noise than other nodes as they belong to classes different from the target node.

	To sum up, we can conclude that successful attacks usually select those adversarial nodes with different classes from the target nodes, low degree, and high predicted confidence. Connecting the specific target node with these nodes usually will lead to a more powerful attack performance than connecting with other nodes.

	\begin{table}[]
		\centering
		\caption{Average degree and confidence (\%) of selected adversarial nodes on three datasets.}
		\begin{tabular}{c|cc|cc}
			\bottomrule\bottomrule
			\multirow{2}{*}{\textbf{Datasets}} & \multicolumn{2}{c|}{\textbf{NMA}} & \multicolumn{2}{c}{\textbf{NMAB}} \\ \cline{2-5} 
			& \multicolumn{1}{c|}{\textbf{Degree}} & \textbf{Confidence}& \multicolumn{1}{c|}{\textbf{Degree}} & \textbf{Confidence} \\ \hline\hline
			Cora     & \multicolumn{1}{c|}{1.66}   & 99.88     & \multicolumn{1}{c|}{1.94}   & 99.91     \\ \hline
			Citeseer & \multicolumn{1}{c|}{1.60}   & 74.64     & \multicolumn{1}{c|}{1.75}   & 76.75     \\ \hline
			Pubmed   & \multicolumn{1}{c|}{1.15}   & 99.64     & \multicolumn{1}{c|}{1.22}   & 99.47     \\ \bottomrule\bottomrule
		\end{tabular}
		\label{tab:dac}
		
	\end{table}

	\begin{table}[]
		\caption{Average time (s) to generate adversarial samples of different methods.}
		\centering\setlength{\tabcolsep}{1.25mm}{
			\begin{tabular}{c|c|c|c|c|c|c}
				\bottomrule\bottomrule
				\textbf{Datasets} & \textbf{NETTACK} & \textbf{FGA} & \textbf{SGA} & \textbf{NGA} & \textbf{NMA} & \textbf{NMAB} \\ \hline\hline
				Cora & 0.880 & 0.103 & 0.022 & 0.004 & 0.237 & 1.226 \\\hline
				Citeseer & 0.676 & 0.093 & 0.042 & 0.003 & 0.188 & 0.904 \\\hline
				Pubmed & 30.852 & 18.016 & 0.287 & 0.010 & 1.431 & 7.294 \\\bottomrule
				\bottomrule
		\end{tabular}}%
		\label{tab:time}
		
	\end{table}

	\subsubsection{Comparison of Searching Space and Time Cost}
	Next, we compare the searching space of the proposed attack strategies with baselines. As we try to select the optimal perturbations to mislead the prediction of target nodes, our issue can be considered as a node/link selection problem. Assuming there are $n$ nodes in the clean graph and the attack budget for the target node is $\Delta$, the theoretical maximal searching space will be close to the combination number $C(n, \Delta)$, which will be a huge number with the increase of $n$. For NETTACK and FGA, as they will greedily select the optimal perturbation among all possible perturbations during each budget, their searching space is close to $\Delta\cdot n$. Particularly, the searching space of NETTACK will be slightly smaller than $\Delta\cdot n$ as it further requires the unnoticeable perturbations. Compared to them, SGA further reduces the searching space by only considering the nodes belonging to the second possible class. If we assume the number of nodes in the second possible class is $|C_s|$, then the searching space of SGA would be close to $\Delta\cdot |C_s|$.

	In terms of the proposed three methods, NGA directly selects the top $\Delta$ links with the highest noise, and thus the searching space of NGA is $\Delta$. For NMA, as we only consider the top $\delta_1$ links with the highest noise value during each budget, the searching space of NMA will be $\Delta\cdot\delta_1$. Finally, for NMAB, we need to evaluate the multiple optimal links under two recorded lists $E_{\rm sin}$ and $E_{\rm optimal}$ whose length are $len_{\rm sin}$ and $len_{\rm re}$, respectively. Therefore, the total searching space of NMAB should be $\delta_2 + len_{\rm sin} \cdot len_{\rm re} \cdot (\Delta-1)$, where the first term means the searching of top $len_{\rm sin}$ links among all $\delta_2$ candidates (i.e., {\em line 8 in Algorithm \ref{alg:nmab}}), and the second term indicates the possible searching space by combining the (1)-link optimal attack list $E_{\rm sin}$ and (i)-link optimal attack list $E_{\rm optimal}$ (i.e., {\em line 12 in Algorithm \ref{alg:nmab}}). Particularly, the searching space of NMAB will equal to NMA once the condition $\delta_1 = \delta_2 = len_{\rm sin} \cdot len_{\rm re}$ is satisfied. Based on the above discussions, as $n \gg |C_s| > \delta$ or $len_{\rm sin}$ or $len_{\rm re}$, the proposed methods will largely reduce searching space while promising a remarkable attack performance.
	
	To better demonstrate the efficiency of the proposed method, we further compare the average time cost for generating the adversarial samples of different attack strategies. As shown in Table \ref{tab:time}, NGA obtains the global optimal results because of its heuristic selections. As expected, NMA performs more efficient than NETTACK as NMA avoids unnecessary searching on the nodes with lower noise. Although NMAB requires relatively larger time cost to generate adversarial samples on Cora and Citeseer, it shows good scalability on Pubmed than NETTACK and FGA. Furthermore, SGA obtains the smallest time cost among all methods, but there are trade-offs between its attack performance and time cost. Combining with previous experiments on attack performance, we can observe that the proposed methods not only yield a strong attack effect but also show remarkable efficiency.

	\subsubsection{Parameters Analysis}

	Finally, we study the attack performance of NMA and NMAB under different parameter combinations. For NMA, we analyze the influence of the size of candidates $\delta_1$. For NMAB, we analyze the size of candidates $\delta_2$, the size of single optimal list $len_{\rm sin}$, and the size of retaining list $len_{\rm re}$. Generally speaking, a larger value of these parameters indicates a larger searching space, which refers to a larger possibility of obtaining the global optimal performance.

	Specifically, we set $\delta_1$ of NMA as $3\Delta$, $5\Delta$, and $10\Delta$. We then set the combination of $\{\delta_2, len_{\rm sin}, len_{\rm re}\}$ of NMAB as $\{10\Delta, 3\Delta, 10\}$, $\{10\Delta, 3\Delta, 2\Delta\}$ and $\{10\Delta, 10\Delta, 3\Delta\}$. We labeled the above 6 groups as G1 to G6. Fig. \ref{fig:pa} shows the corresponding results of different settings of the above parameters of NMA and NMAB. For NMA (i.e., G1-G3), a larger value of $\delta_1$ will truly lead to better performance in most cases, and NMA can achieve an excellent performance even when $\delta_1$ is a small value. In other words, we can achieve a good attack performance by only considering the candidate nodes with higher noise. The possible reason for the unstable performance of G3 on Citeseer is that NMA may drop into the local optimal in the early steps. As for NMAB (i.e., G4-G6), a similar trend can be obtained. NMAB also achieves a strong attack performance by limiting the searching space to a relatively small value via the noise value. The above observations support that the proposed noise can help fast localize the powerful adversarial nodes that negatively influence the aggregations of target nodes the most. Therefore, our candidate selection mechanism based on noise can be considered as a plugin to integrate into further attack methods in the pre-processing step, which can omit some unnecessary searching on those poorly-performed candidates.

	\section{Conclusion}\label{sec:c}

	In this work, we theoretically discuss the attack strength of different adversarial structural perturbations of graph neural networks, and then put forward the concept of noise to characterize them. By quantifying the noise value of adversarial links via entropy, we further propose three simple yet effective targeted attack strategies, namely the noise-based method (NGA), the noise and margin-based method (NMA), and the boost version of the noise and margin-based method considering multiple budgets at the same time (NMAB). The latter two methods can greatly reduce the searching space of traditional margin-based methods while yielding a strong attack effect. Comprehensive experimental results against various graph neural network models on the benchmark datasets demonstrate the superiority of the proposed methods. Particularly, the analysis of the properties of selected adversarial nodes also supports the effectiveness of the proposed noise concept. Theoretical proof on the toy model, together with extensive empirical experiments, also shows the rationality of the proposed noise concept. Future attack methods can integrate the proposed candidate refining framework to avoid unnecessary searching for perturbations with small noise. In addition, there are some interesting directions that need further investigation in the future. Firstly, we will investigate whether the proposed attack strategies can be applied to large-scale graphs to improve the generality. Secondly, we will explore how to ensure the effectiveness of our attack strategies against defense methods.

\begin{acks}
\end{acks}

\bibliographystyle{ACM-Reference-Format}
\bibliography{mybib}


\begin{thebibliography}{33}


\ifx \showCODEN    \undefined \def \showCODEN     #1{\unskip}     \fi
\ifx \showISBNx    \undefined \def \showISBNx     #1{\unskip}     \fi
\ifx \showISBNxiii \undefined \def \showISBNxiii  #1{\unskip}     \fi
\ifx \showISSN     \undefined \def \showISSN      #1{\unskip}     \fi
\ifx \showLCCN     \undefined \def \showLCCN      #1{\unskip}     \fi
\ifx \shownote     \undefined \def \shownote      #1{#1}          \fi
\ifx \showarticletitle \undefined \def \showarticletitle #1{#1}   \fi
\ifx \showURL      \undefined \def \showURL       {\relax}        \fi
\providecommand\bibfield[2]{#2}
\providecommand\bibinfo[2]{#2}
\providecommand\natexlab[1]{#1}
\providecommand\showeprint[2][]{arXiv:#2}

\bibitem[Alom et~al\mbox{.}(2025)]%
        {alomgottack}
\bibfield{author}{\bibinfo{person}{Zulfikar Alom}, \bibinfo{person}{Tran Gia~Bao Ngo}, \bibinfo{person}{Murat Kantarcioglu}, {and} \bibinfo{person}{Cuneyt~Gurcan Akcora}.} \bibinfo{year}{2025}\natexlab{}.
\newblock \showarticletitle{GOttack: Universal Adversarial Attacks on Graph Neural Networks via Graph Orbits Learning}. In \bibinfo{booktitle}{\emph{Proc. ICLR}}.
\newblock


\bibitem[Bojchevski and G{\"u}nnemann(2019)]%
        {bojchevski2019certifiable}
\bibfield{author}{\bibinfo{person}{Aleksandar Bojchevski} {and} \bibinfo{person}{Stephan G{\"u}nnemann}.} \bibinfo{year}{2019}\natexlab{}.
\newblock \showarticletitle{Certifiable robustness to graph perturbations}.
\newblock \bibinfo{journal}{\emph{Proc. NeurIPS}}  \bibinfo{volume}{32} (\bibinfo{year}{2019}).
\newblock


\bibitem[Chen et~al\mbox{.}(2019)]%
        {chen2019ga}
\bibfield{author}{\bibinfo{person}{Jinyin Chen}, \bibinfo{person}{Lihong Chen}, \bibinfo{person}{Yixian Chen}, \bibinfo{person}{Minghao Zhao}, \bibinfo{person}{Shanqing Yu}, \bibinfo{person}{Qi Xuan}, {and} \bibinfo{person}{Xiaoniu Yang}.} \bibinfo{year}{2019}\natexlab{}.
\newblock \showarticletitle{GA-based q-attack on community detection}.
\newblock \bibinfo{journal}{\emph{IEEE Trans. Comput. Soc. Syst.}} \bibinfo{volume}{6}, \bibinfo{number}{3} (\bibinfo{year}{2019}), \bibinfo{pages}{491--503}.
\newblock


\bibitem[Chen et~al\mbox{.}(2018)]%
        {chen2018fast}
\bibfield{author}{\bibinfo{person}{Jinyin Chen}, \bibinfo{person}{Yangyang Wu}, \bibinfo{person}{Xuanheng Xu}, \bibinfo{person}{Yixian Chen}, \bibinfo{person}{Haibin Zheng}, {and} \bibinfo{person}{Qi Xuan}.} \bibinfo{year}{2018}\natexlab{}.
\newblock \showarticletitle{Fast gradient attack on network embedding}.
\newblock \bibinfo{journal}{\emph{arXiv preprint arXiv:1809.02797}} (\bibinfo{year}{2018}).
\newblock


\bibitem[Chen et~al\mbox{.}(2020)]%
        {chen2020survey}
\bibfield{author}{\bibinfo{person}{Liang Chen}, \bibinfo{person}{Jintang Li}, \bibinfo{person}{Jiaying Peng}, \bibinfo{person}{Tao Xie}, \bibinfo{person}{Zengxu Cao}, \bibinfo{person}{Kun Xu}, \bibinfo{person}{Xiangnan He}, {and} \bibinfo{person}{Zibin Zheng}.} \bibinfo{year}{2020}\natexlab{}.
\newblock \showarticletitle{A survey of adversarial learning on graphs}.
\newblock \bibinfo{journal}{\emph{arXiv preprint arXiv:2003.05730}} (\bibinfo{year}{2020}).
\newblock


\bibitem[Dai et~al\mbox{.}(2018)]%
        {dai2018adversarial}
\bibfield{author}{\bibinfo{person}{Hanjun Dai}, \bibinfo{person}{Hui Li}, \bibinfo{person}{Tian Tian}, \bibinfo{person}{Xin Huang}, \bibinfo{person}{Lin Wang}, \bibinfo{person}{Jun Zhu}, {and} \bibinfo{person}{Le Song}.} \bibinfo{year}{2018}\natexlab{}.
\newblock \showarticletitle{Adversarial attack on graph structured data}. In \bibinfo{booktitle}{\emph{Proc. ICML}}. PMLR, \bibinfo{pages}{1115--1124}.
\newblock


\bibitem[Fang et~al\mbox{.}(2024)]%
        {fang2024gani}
\bibfield{author}{\bibinfo{person}{Junyuan Fang}, \bibinfo{person}{Haixian Wen}, \bibinfo{person}{Jiajing Wu}, \bibinfo{person}{Qi Xuan}, \bibinfo{person}{Zibin Zheng}, {and} \bibinfo{person}{Chi~K Tse}.} \bibinfo{year}{2024}\natexlab{}.
\newblock \showarticletitle{Gani: Global attacks on graph neural networks via imperceptible node injections}.
\newblock \bibinfo{journal}{\emph{IEEE Trans. Comput. Soc. Syst.}} \bibinfo{volume}{11}, \bibinfo{number}{4} (\bibinfo{year}{2024}), \bibinfo{pages}{5374--5387}.
\newblock


\bibitem[Fang et~al\mbox{.}(2025)]%
        {fang2025contributes}
\bibfield{author}{\bibinfo{person}{Junyuan Fang}, \bibinfo{person}{Han Yang}, \bibinfo{person}{Jiajing Wu}, \bibinfo{person}{Zibin Zheng}, {and} \bibinfo{person}{Chi~K Tse}.} \bibinfo{year}{2025}\natexlab{}.
\newblock \showarticletitle{What Contributes More to the Robustness of Heterophilic Graph Neural Networks?}
\newblock \bibinfo{journal}{\emph{IEEE Trans. Syst. Man Cybern.: Syst.}} (\bibinfo{year}{2025}).
\newblock


\bibitem[Geisler et~al\mbox{.}(2021)]%
        {geisler2021robustness}
\bibfield{author}{\bibinfo{person}{Simon Geisler}, \bibinfo{person}{Tobias Schmidt}, \bibinfo{person}{Hakan {\c{S}}irin}, \bibinfo{person}{Daniel Z{\"u}gner}, \bibinfo{person}{Aleksandar Bojchevski}, {and} \bibinfo{person}{Stephan G{\"u}nnemann}.} \bibinfo{year}{2021}\natexlab{}.
\newblock \showarticletitle{Robustness of graph neural networks at scale}.
\newblock \bibinfo{journal}{\emph{Proc. NeurIPS}}  \bibinfo{volume}{34} (\bibinfo{year}{2021}), \bibinfo{pages}{7637--7649}.
\newblock


\bibitem[Hamilton et~al\mbox{.}(2017)]%
        {hamilton2017inductive}
\bibfield{author}{\bibinfo{person}{Will Hamilton}, \bibinfo{person}{Zhitao Ying}, {and} \bibinfo{person}{Jure Leskovec}.} \bibinfo{year}{2017}\natexlab{}.
\newblock \showarticletitle{Inductive representation learning on large graphs}.
\newblock \bibinfo{journal}{\emph{Proc. NeurIPS}}  \bibinfo{volume}{30} (\bibinfo{year}{2017}).
\newblock


\bibitem[Jiang and Luo(2021)]%
        {jiang2021graph}
\bibfield{author}{\bibinfo{person}{Weiwei Jiang} {and} \bibinfo{person}{Jiayun Luo}.} \bibinfo{year}{2021}\natexlab{}.
\newblock \showarticletitle{Graph neural network for traffic forecasting: A survey}.
\newblock \bibinfo{journal}{\emph{arXiv preprint arXiv:2101.11174}} (\bibinfo{year}{2021}).
\newblock


\bibitem[Jin et~al\mbox{.}(2020)]%
        {jin2020adversarial}
\bibfield{author}{\bibinfo{person}{Wei Jin}, \bibinfo{person}{Yaxin Li}, \bibinfo{person}{Han Xu}, \bibinfo{person}{Yiqi Wang}, \bibinfo{person}{Shuiwang Ji}, \bibinfo{person}{Charu Aggarwal}, {and} \bibinfo{person}{Jiliang Tang}.} \bibinfo{year}{2020}\natexlab{}.
\newblock \showarticletitle{Adversarial attacks and defenses on graphs: A review, a tool and empirical studies}.
\newblock \bibinfo{journal}{\emph{arXiv preprint arXiv:2003.00653}} (\bibinfo{year}{2020}).
\newblock


\bibitem[Kipf and Welling(2017)]%
        {kipf2016semi}
\bibfield{author}{\bibinfo{person}{Thomas~N. Kipf} {and} \bibinfo{person}{Max Welling}.} \bibinfo{year}{2017}\natexlab{}.
\newblock \showarticletitle{Semi-supervised classification with graph convolutional networks}. In \bibinfo{booktitle}{\emph{Proc. ICLR}}. \bibinfo{publisher}{OpenReview.net}.
\newblock


\bibitem[Li et~al\mbox{.}(2021a)]%
        {li2021deep}
\bibfield{author}{\bibinfo{person}{Jintang Li}, \bibinfo{person}{Zishan Gu}, \bibinfo{person}{Qibiao Peng}, \bibinfo{person}{Kun Xu}, \bibinfo{person}{Liang Chen}, {and} \bibinfo{person}{Zibin Zheng}.} \bibinfo{year}{2021}\natexlab{a}.
\newblock \showarticletitle{Deep insights into graph adversarial learning: An empirical study perspective}. In \bibinfo{booktitle}{\emph{Proc. IJCAI-HBAI}}. Springer, \bibinfo{pages}{87--101}.
\newblock


\bibitem[Li et~al\mbox{.}(2021b)]%
        {li2021adversarial}
\bibfield{author}{\bibinfo{person}{Jintang Li}, \bibinfo{person}{Tao Xie}, \bibinfo{person}{Chen Liang}, \bibinfo{person}{Fenfang Xie}, \bibinfo{person}{Xiangnan He}, {and} \bibinfo{person}{Zibin Zheng}.} \bibinfo{year}{2021}\natexlab{b}.
\newblock \showarticletitle{Adversarial attack on large scale graph}.
\newblock \bibinfo{journal}{\emph{IEEE Trans. Knowl. Data Eng.}} (\bibinfo{year}{2021}).
\newblock


\bibitem[Pei et~al\mbox{.}(2020)]%
        {pei2020geom}
\bibfield{author}{\bibinfo{person}{Hongbin Pei}, \bibinfo{person}{Bingzhe Wei}, \bibinfo{person}{Kevin~Chen{-}Chuan Chang}, \bibinfo{person}{Yu Lei}, {and} \bibinfo{person}{Bo Yang}.} \bibinfo{year}{2020}\natexlab{}.
\newblock \showarticletitle{Geom-GCN: Geometric graph convolutional networks}. In \bibinfo{booktitle}{\emph{Proc. ICLR}}.
\newblock


\bibitem[Sen et~al\mbox{.}(2008)]%
        {sen2008collective}
\bibfield{author}{\bibinfo{person}{Prithviraj Sen}, \bibinfo{person}{Galileo Namata}, \bibinfo{person}{Mustafa Bilgic}, \bibinfo{person}{Lise Getoor}, \bibinfo{person}{Brian Galligher}, {and} \bibinfo{person}{Tina Eliassi-Rad}.} \bibinfo{year}{2008}\natexlab{}.
\newblock \showarticletitle{Collective classification in network data}.
\newblock \bibinfo{journal}{\emph{AI Mag.}} \bibinfo{volume}{29}, \bibinfo{number}{3} (\bibinfo{year}{2008}), \bibinfo{pages}{93--93}.
\newblock


\bibitem[Sun et~al\mbox{.}(2018)]%
        {sun2018adversarial}
\bibfield{author}{\bibinfo{person}{Lichao Sun}, \bibinfo{person}{Yingtong Dou}, \bibinfo{person}{Carl Yang}, \bibinfo{person}{Ji Wang}, \bibinfo{person}{Philip~S Yu}, \bibinfo{person}{Lifang He}, {and} \bibinfo{person}{Bo Li}.} \bibinfo{year}{2018}\natexlab{}.
\newblock \showarticletitle{Adversarial attack and defense on graph data: A survey}.
\newblock \bibinfo{journal}{\emph{arXiv preprint arXiv:1812.10528}} (\bibinfo{year}{2018}).
\newblock


\bibitem[Sun et~al\mbox{.}(2020)]%
        {sun2019node}
\bibfield{author}{\bibinfo{person}{Yiwei Sun}, \bibinfo{person}{Suhang Wang}, \bibinfo{person}{Xianfeng Tang}, \bibinfo{person}{Tsung-Yu Hsieh}, {and} \bibinfo{person}{Vasant Honavar}.} \bibinfo{year}{2020}\natexlab{}.
\newblock \showarticletitle{Adversarial attacks on graph neural networks via node injections: A hierarchical reinforcement learning approach}. In \bibinfo{booktitle}{\emph{Proc. WWW}}. \bibinfo{pages}{673–683}.
\newblock


\bibitem[Tao et~al\mbox{.}(2021)]%
        {tao2021single}
\bibfield{author}{\bibinfo{person}{Shuchang Tao}, \bibinfo{person}{Qi Cao}, \bibinfo{person}{Huawei Shen}, \bibinfo{person}{Junjie Huang}, \bibinfo{person}{Yunfan Wu}, {and} \bibinfo{person}{Xueqi Cheng}.} \bibinfo{year}{2021}\natexlab{}.
\newblock \showarticletitle{Single node injection attack against graph neural networks}. In \bibinfo{booktitle}{\emph{Proc. CIKM}}. \bibinfo{pages}{1794--1803}.
\newblock


\bibitem[Veli{\v{c}}kovi{\'c} et~al\mbox{.}(2018)]%
        {velivckovic2017graph}
\bibfield{author}{\bibinfo{person}{Petar Veli{\v{c}}kovi{\'c}}, \bibinfo{person}{Guillem Cucurull}, \bibinfo{person}{Arantxa Casanova}, \bibinfo{person}{Adriana Romero}, \bibinfo{person}{Pietro Lio}, {and} \bibinfo{person}{Yoshua Bengio}.} \bibinfo{year}{2018}\natexlab{}.
\newblock \showarticletitle{Graph attention networks}. In \bibinfo{booktitle}{\emph{Proc. ICLR}}.
\newblock


\bibitem[Waniek et~al\mbox{.}(2018)]%
        {waniek2018hiding}
\bibfield{author}{\bibinfo{person}{Marcin Waniek}, \bibinfo{person}{Tomasz~P Michalak}, \bibinfo{person}{Michael~J Wooldridge}, {and} \bibinfo{person}{Talal Rahwan}.} \bibinfo{year}{2018}\natexlab{}.
\newblock \showarticletitle{Hiding individuals and communities in a social network}.
\newblock \bibinfo{journal}{\emph{Nat. Hum. Behav.}} \bibinfo{volume}{2}, \bibinfo{number}{2} (\bibinfo{year}{2018}), \bibinfo{pages}{139--147}.
\newblock


\bibitem[Wu et~al\mbox{.}(2019a)]%
        {wu2019simplifying}
\bibfield{author}{\bibinfo{person}{Felix Wu}, \bibinfo{person}{Amauri Souza}, \bibinfo{person}{Tianyi Zhang}, \bibinfo{person}{Christopher Fifty}, \bibinfo{person}{Tao Yu}, {and} \bibinfo{person}{Kilian Weinberger}.} \bibinfo{year}{2019}\natexlab{a}.
\newblock \showarticletitle{Simplifying graph convolutional networks}. In \bibinfo{booktitle}{\emph{Proc. ICML}}. PMLR, \bibinfo{pages}{6861--6871}.
\newblock


\bibitem[Wu et~al\mbox{.}(2019b)]%
        {wu2019adversarial}
\bibfield{author}{\bibinfo{person}{Huijun Wu}, \bibinfo{person}{Chen Wang}, \bibinfo{person}{Yuriy Tyshetskiy}, \bibinfo{person}{Andrew Docherty}, \bibinfo{person}{Kai Lu}, {and} \bibinfo{person}{Liming Zhu}.} \bibinfo{year}{2019}\natexlab{b}.
\newblock \showarticletitle{Adversarial examples for graph data: Deep insights into attack and defense}. In \bibinfo{booktitle}{\emph{Proc. IJCAI}}. \bibinfo{pages}{4816--4823}.
\newblock


\bibitem[Wu et~al\mbox{.}(2020b)]%
        {wu2020graph}
\bibfield{author}{\bibinfo{person}{Shiwen Wu}, \bibinfo{person}{Fei Sun}, \bibinfo{person}{Wentao Zhang}, \bibinfo{person}{Xu Xie}, {and} \bibinfo{person}{Bin Cui}.} \bibinfo{year}{2020}\natexlab{b}.
\newblock \showarticletitle{Graph neural networks in recommender systems: a survey}.
\newblock \bibinfo{journal}{\emph{ACM Comput. Surv.}} (\bibinfo{year}{2020}).
\newblock


\bibitem[Wu et~al\mbox{.}(2020a)]%
        {wu2020comprehensive}
\bibfield{author}{\bibinfo{person}{Zonghan Wu}, \bibinfo{person}{Shirui Pan}, \bibinfo{person}{Fengwen Chen}, \bibinfo{person}{Guodong Long}, \bibinfo{person}{Chengqi Zhang}, {and} \bibinfo{person}{S~Yu Philip}.} \bibinfo{year}{2020}\natexlab{a}.
\newblock \showarticletitle{A comprehensive survey on graph neural networks}.
\newblock \bibinfo{journal}{\emph{IEEE Trans. Neural Netw. Learn. Syst.}} \bibinfo{volume}{32}, \bibinfo{number}{1} (\bibinfo{year}{2020}), \bibinfo{pages}{4--24}.
\newblock


\bibitem[Zhang et~al\mbox{.}(2020)]%
        {zhang2020cross}
\bibfield{author}{\bibinfo{person}{Qifan Zhang}, \bibinfo{person}{Junyuan Fang}, \bibinfo{person}{Jie Zhang}, \bibinfo{person}{Jiajing Wu}, \bibinfo{person}{Yongxiang Xia}, {and} \bibinfo{person}{Zibin Zheng}.} \bibinfo{year}{2020}\natexlab{}.
\newblock \showarticletitle{Cross entropy attack on deep graph infomax}. In \bibinfo{booktitle}{\emph{Proc. ISCAS}}. IEEE, \bibinfo{pages}{1--5}.
\newblock


\bibitem[Zhou et~al\mbox{.}(2020)]%
        {zhou2020graph}
\bibfield{author}{\bibinfo{person}{Jie Zhou}, \bibinfo{person}{Ganqu Cui}, \bibinfo{person}{Shengding Hu}, \bibinfo{person}{Zhengyan Zhang}, \bibinfo{person}{Cheng Yang}, \bibinfo{person}{Zhiyuan Liu}, \bibinfo{person}{Lifeng Wang}, \bibinfo{person}{Changcheng Li}, {and} \bibinfo{person}{Maosong Sun}.} \bibinfo{year}{2020}\natexlab{}.
\newblock \showarticletitle{Graph neural networks: A review of methods and applications}.
\newblock \bibinfo{journal}{\emph{AI Open}}  \bibinfo{volume}{1} (\bibinfo{year}{2020}), \bibinfo{pages}{57--81}.
\newblock


\bibitem[Zhu et~al\mbox{.}(2024)]%
        {zhu2024simple}
\bibfield{author}{\bibinfo{person}{Guanghui Zhu}, \bibinfo{person}{Mengyu Chen}, \bibinfo{person}{Chunfeng Yuan}, {and} \bibinfo{person}{Yihua Huang}.} \bibinfo{year}{2024}\natexlab{}.
\newblock \showarticletitle{Simple and efficient partial graph adversarial attack: A new perspective}.
\newblock \bibinfo{journal}{\emph{IEEE Trans. Knowl. Data Eng.}} (\bibinfo{year}{2024}).
\newblock


\bibitem[Zhu et~al\mbox{.}(2020)]%
        {zhu2020beyond}
\bibfield{author}{\bibinfo{person}{Jiong Zhu}, \bibinfo{person}{Yujun Yan}, \bibinfo{person}{Lingxiao Zhao}, \bibinfo{person}{Mark Heimann}, \bibinfo{person}{Leman Akoglu}, {and} \bibinfo{person}{Danai Koutra}.} \bibinfo{year}{2020}\natexlab{}.
\newblock \showarticletitle{Beyond homophily in graph neural networks: Current limitations and effective designs}. In \bibinfo{booktitle}{\emph{Proc. NeurIPS}}.
\newblock


\bibitem[Zhu et~al\mbox{.}(2014)]%
        {zhu2014revealing}
\bibfield{author}{\bibinfo{person}{Yihai Zhu}, \bibinfo{person}{Jun Yan}, \bibinfo{person}{Yan Sun}, {and} \bibinfo{person}{Haibo He}.} \bibinfo{year}{2014}\natexlab{}.
\newblock \showarticletitle{Revealing cascading failure vulnerability in power grids using risk-graph}.
\newblock \bibinfo{journal}{\emph{IEEE Trans. Parallel Distrib. Syst.}} \bibinfo{volume}{25}, \bibinfo{number}{12} (\bibinfo{year}{2014}), \bibinfo{pages}{3274--3284}.
\newblock


\bibitem[Z{\"u}gner et~al\mbox{.}(2018)]%
        {zugner2018adversarial}
\bibfield{author}{\bibinfo{person}{Daniel Z{\"u}gner}, \bibinfo{person}{Amir Akbarnejad}, {and} \bibinfo{person}{Stephan G{\"u}nnemann}.} \bibinfo{year}{2018}\natexlab{}.
\newblock \showarticletitle{Adversarial attacks on neural networks for graph data}. In \bibinfo{booktitle}{\emph{Proc. KDD}}. \bibinfo{pages}{2847--2856}.
\newblock


\bibitem[Z{\"u}gner and G{\"u}nnemann(2019)]%
        {zugner2019adversarial}
\bibfield{author}{\bibinfo{person}{Daniel Z{\"u}gner} {and} \bibinfo{person}{Stephan G{\"u}nnemann}.} \bibinfo{year}{2019}\natexlab{}.
\newblock \showarticletitle{Adversarial attacks on graph neural networks via meta learning}. In \bibinfo{booktitle}{\emph{Proc. ICLR}}.
\newblock


\end{thebibliography}

\appendix
\section{Proof of Propositions}

\subsection{Proof of Proposition 4.1}

\begin{proposition} 
    Let $G = (A, X, E)$ be a simple graph, and $Y = \{0,1,\cdots,C-1\}$ be the possible label. We simplify the feature of each node to be a one-hot vector corresponding to the label of itself, denoted as $\mu(Y)$. Namely, the feature vector of node $u$ is $x_{u} = \mu(Y_u)$. Assuming that most of the original neighbors of each node belong to the same class, and the specific noise value of each adversarial link is the same. Consider a one-layer GCN where the output of node $u$ is $ h_u = \sigma(W \cdot \sum_{v \in \mathcal{N}(u)} \frac{1}{\sqrt{|\mathcal{N}_u| \cdot |\mathcal{N}_v|}} \cdot x_v)$, $\sigma$ is the softmax activation function, we have the following. 
    \begin{enumerate}
        \item From the perspective of target nodes, nodes with a lower degree will be easier to be attacked than those with a higher degree.
        \item From the perspective of adversarial nodes, nodes with a lower degree will influence the representation of the target node more than those with a higher degree.
    \end{enumerate} 
\end{proposition}

\label{proof:4.1}

\begin{proof}

As we can know, $W \in \mathbb{R}^{C \times C}$ is the learning parameter that needs to be optimized in our assumption. Since we only utilize a single layer GCN, the aggregated features from the aggregator are the weighted sum of original features (i.e., one-hot vectors corresponding to their labels). Moreover, we can easily obtain that, a well-performed GCN can be trained with the optimal $W^*$ which is similar to follows. $W^*$ would be a diagonal-like matrix where the diagonal elements are non-zero while all other elements are mostly close to zero. Only under this condition, the final output vector $h_u$ will have a large value on $u$'s corresponding label index, indicating that the corresponding GCN can classify the nodes to the ground truth label with a large probability.

For a specific node $u$ with $k$ neighbors, the original output $h_u$ is given as follows.

\begin{equation} \label{eq:gen_hu}
    \begin{aligned}
        h_u &= \sigma(W^* \cdot \sum_{v \in \mathcal{N}(u)} \frac{1}{\sqrt{|\mathcal{N}_u| \cdot |\mathcal{N}_v|}} \cdot x_v) \\
        &= \sigma[\frac{W^*}{\sqrt{|\mathcal{N}_u|}} \cdot (\frac{x_i}{\sqrt{|\mathcal{N}_i|}} + \frac{x_j}{\sqrt{|\mathcal{N}_j|}} +  \cdots  \\ & \quad + \frac{x_k}{\sqrt{|\mathcal{N}_k|}})]
    \end{aligned}
\end{equation} 
Therefore, for the first claim of Proposition \ref{pro1}, assuming we have another targeted node $a$ having neighbors $l, m, ..., n$. The label of node $a$ is the same as $u$ (i.e., $Y_a = Y_u$) and the degree of node $a$ is higher than the degree of node $u$ (i.e., $|\mathcal{N}_a| > |\mathcal{N}_u|$), which can be given as follows.

\begin{equation} \label{gen_ha}
    h_a = \sigma[\frac{W^*}{\sqrt{|\mathcal{N}_a|}} \cdot (\frac{x_l}{\sqrt{|\mathcal{N}_l|}} + \frac{x_m}{\sqrt{|\mathcal{N}_m|}} + \cdots + \frac{x_n}{\sqrt{|\mathcal{N}_n|}})]
\end{equation}

Then, after we connect the same adversarial node $e$ to the target nodes, respectively, the latest output of GCN under $W^*$ will be as follows.

\begin{equation} \label{eq:gen_hua}
    \begin{aligned}
        h_u &= \sigma[\frac{W^*}{\sqrt{|\mathcal{N}_u| + 1}} \cdot (\frac{x_i}{\sqrt{|\mathcal{N}_i|}} + \frac{x_j}{\sqrt{|\mathcal{N}_j|}} + \cdots  \\ & \quad+ \frac{x_k}{\sqrt{|\mathcal{N}_k|}} + \frac{x_e}{\sqrt{|\mathcal{N}_e|+1}})] \\ \\
        h_a &= \sigma[\frac{W^*}{\sqrt{|\mathcal{N}_a| + 1}} \cdot (\frac{x_l}{\sqrt{|\mathcal{N}_l|}} + \frac{x_m}{\sqrt{|\mathcal{N}_m|}} + \cdots  \\ & \quad+ \frac{x_n}{\sqrt{|\mathcal{N}_n|}} + \frac{x_e}{\sqrt{|\mathcal{N}_e|+1}})],
        \end{aligned}
\end{equation} 
where $|\mathcal{N}_e|$ is the original degree of node $e$.

Since we assume the label of original neighbors (i.e., $x_i, x_j, \cdots, x_k$) of node $u$ are the same as center node, we can utilize one-hot feature $\mu(Y_u)$ to simplify (\ref{eq:gen_hua}). Moreover, if we further assume the degree of all neighbors equal to the average degree of the original graph, denoted as $\!<\!d\!>\!$, we can further have follows.

\begin{equation} \label{eq:gen_hua1}
    \begin{aligned}
        h_u &= \sigma[\frac{W^*}{\sqrt{|\mathcal{N}_u| + 1}} \cdot (\underbrace{\frac{\mu(Y_u)}{\sqrt{\!<\!d\!>\!}} + \frac{\mu(Y_u)}{\sqrt{\!<\!d\!>\!}} + \cdots + \frac{\mu(Y_u)}{\sqrt{\!<\!d\!>\!}}}_{|\mathcal{N}_u|}  \\ & \quad+ \frac{\mu(Y_e)}{\sqrt{|\mathcal{N}_e|+1}})] \\
        &= \sigma[\frac{W^*}{\sqrt{|\mathcal{N}_u| + 1}} \cdot (\frac{|\mathcal{N}_u| \cdot \mu(Y_u)}{\sqrt{\!<\!d\!>\!}} + \frac{\mu(Y_e)}{\sqrt{|\mathcal{N}_e|+1}})] \\
        &= \sigma[\frac{|\mathcal{N}_u| \cdot W^* \cdot \mu(Y_u)}{\sqrt{(|\mathcal{N}_u| + 1) \cdot \!<\!d\!>\!}} + \frac{W^* \cdot \mu(Y_e)}{\sqrt{(|\mathcal{N}_u| + 1) \cdot (|\mathcal{N}_e|+1)}}] \\ \\
        h_a &= \sigma[\frac{W^*}{\sqrt{|\mathcal{N}_a| + 1}} \cdot (\underbrace{\frac{\mu(Y_a)}{\sqrt{\!<\!d\!>\!}} + \frac{\mu(Y_a)}{\sqrt{\!<\!d\!>\!}} + \cdots + \frac{\mu(Y_a)}{\sqrt{\!<\!d\!>\!}}}_{|\mathcal{N}_a|}  \\ & \quad+ \frac{\mu(Y_e)}{\sqrt{|\mathcal{N}_e|+1}})] \\
        &= \sigma[\frac{W^*}{\sqrt{|\mathcal{N}_a| + 1}} \cdot (\frac{|\mathcal{N}_a| \cdot \mu(Y_a)}{\sqrt{\!<\!d\!>\!}} + \frac{\mu(Y_e)}{\sqrt{|\mathcal{N}_e|+1}})] \\
        &= \sigma[\frac{|\mathcal{N}_a| \cdot W^* \cdot \mu(Y_a)}{\sqrt{(|\mathcal{N}_a| + 1) \cdot \!<\!d\!>\!}} + \frac{W^* \cdot \mu(Y_e)}{\sqrt{(|\mathcal{N}_a| + 1) \cdot (|\mathcal{N}_e|+1)}}]
    \end{aligned}
\end{equation}
As $|\mathcal{N}_a| > |\mathcal{N}_u|$, we can easily obtain the following from (\ref{eq:gen_hua1}).

\begin{equation} \label{eq:gen_hua2}
    \begin{aligned}
        \frac{|\mathcal{N}_u|}{\sqrt{(|\mathcal{N}_u| + 1) \cdot \!<\!d\!>\!}} &- \frac{|\mathcal{N}_a|}{\sqrt{(|\mathcal{N}_a| + 1) \cdot \!<\!d\!>\!}} < 0\\ \\
        \frac{1}{\sqrt{(|\mathcal{N}_u| + 1) \cdot (|\mathcal{N}_e| + 1)}} &- \frac{1}{\sqrt{(|\mathcal{N}_a| + 1) \cdot (|\mathcal{N}_e| + 1)}} > 0 
    \end{aligned}	
\end{equation}
Therefore, after multiplying with the diagonal-like optimal weight matrix $W^*$, and the final adjustment of the monotonic softmax activation function, the coefficient in each term represents the specific weight of the corresponding label. For the first term of nodes $u$ and $a$ in (\ref{eq:gen_hua1}), we can easily obtain that for the original features denoted by $\mu(Y_u)$ or $\mu(Y_a)$, node $a$ will assign a larger weight than node $u$. On the contrary, for the second term of nodes $u$ and $a$ in (\ref{eq:gen_hua1}), we can obtain that for the adversarial features $\mu(Y_e)$ induced by the adversarial link, node $a$ will assign a smaller weight than node $u$. Therefore, we have $h_{a}[Y_u] > h_{u}[Y_u]$. That is to say, in the $Y_u$-th/$Y_a$-th index of $h_{a}/h_{u}$, the value of node $a$ will larger than node $u$. Therefore, for the target nodes, nodes with a lower degree are easier to be influenced than those with a higher degree.

For the second claim of Proposition \ref{pro1}, we can proof following the similar procedures as before. To attack the target node $u$, assuming we have two different adversarial nodes $p$ and $q$ where $|\mathcal{N}_p| > |\mathcal{N}_q|$. Then, for the latest output of node $u$ after attacking by nodes $p$ and $q$, respectively, we can have follows.

\begin{equation} \label{eq:gen_huu}
    \begin{aligned}
        h_{u_p} &= \sigma[\frac{W^*}{\sqrt{|\mathcal{N}_u| + 1}} \cdot (\frac{x_i}{\sqrt{|\mathcal{N}_i|}} + \frac{x_j}{\sqrt{|\mathcal{N}_j|}} + \cdots + \frac{x_k}{\sqrt{|\mathcal{N}_k|}}  \\ & \quad+ \frac{x_p}{\sqrt{|\mathcal{N}_p|+1}})] \\ \\
        h_{u_q} &= \sigma[\frac{W^*}{\sqrt{|\mathcal{N}_u| + 1}} \cdot (\frac{x_i}{\sqrt{|\mathcal{N}_i|}} + \frac{x_j}{\sqrt{|\mathcal{N}_j|}} + \cdots + \frac{x_k}{\sqrt{|\mathcal{N}_k|}}  \\ & \quad+ \frac{x_q}{\sqrt{|\mathcal{N}_q|+1}})]
    \end{aligned}
\end{equation} 
Similar to (\ref{eq:gen_hua1}), we can simplify (\ref{eq:gen_huu}) under the same assumptions, which is as follows.
\begin{equation} \label{eq:gen_huu1}
    \begin{aligned}
        h_{u_p} &= \sigma[\frac{|\mathcal{N}_u| \cdot W^* \cdot \mu(Y_u)}{\sqrt{(|\mathcal{N}_u| + 1) \cdot \!<\!d\!>\!}} + \frac{W^* \cdot \mu(Y_p)}{\sqrt{(|\mathcal{N}_u| + 1) \cdot (|\mathcal{N}_p|+1)}}] \\ \\
        h_{u_q} &= \sigma[\frac{|\mathcal{N}_u| \cdot W^* \cdot \mu(Y_u)}{\sqrt{(|\mathcal{N}_u| + 1) \cdot \!<\!d\!>\!}} + \frac{W^* \cdot \mu(Y_q)}{\sqrt{(|\mathcal{N}_u| + 1) \cdot (|\mathcal{N}_q|+1)}}]
    \end{aligned}
\end{equation}
Similarly, after multiplying with the diagonal-like optimal weight matrix $W^*$ and the adjustment of the monotonic softmax function, the coefficient of each term represents the specific weight of the corresponding possible label. From (\ref{eq:gen_huu1}), we know that the first term of $h_{u_p}$ and $h_{u_q}$ is the same. While for the second term, the latter situation (i.e., connects with node $q$) will assign a larger weight to the adversarial nodes/features than the first situation (i.e., connects with node $p$) since $|\mathcal{N}_p| > |\mathcal{N}_q|$, so we have $h_{u_p}[Y_u] > h_{u_q}[Y_u]$.  That is to say, in the $Y_u$-th index of final outputs, the value of $u_p$ will larger than $u_q$. Therefore, for the adversarial nodes, nodes with a lower degree can influence the aggregation of the target node more than those with a higher degree.
\end{proof}

\subsection{Proof of Proposition 4.2}

\begin{proposition}
    Except for the specific noise value of adversarial links varying from each other, we let all of the other assumptions be the same as Proposition \ref{pro1}. Then, we have the following. For a specific target node $u$, if the adversarial nodes have the same degree, the adversarial nodes which are dissimilar to node $u$ influence the aggregation of node $u$ more than those similar to node $u$.
\end{proposition}

\label{proof:4.2}

\begin{proof}
		Assuming we have two different adversarial nodes $p$, $q$, and $|\mathcal{N}_p| = |\mathcal{N}_q|$, but their similarities with the target node $u$ are different. Specifically, ${\rm SIM}(u,p) > {\rm SIM}(u,q)$ where ${\rm SIM}(\cdot, \cdot)$ means the similarity function. 
		\begin{equation} \label{eq:gen_huv}
			\begin{aligned}
				h_{u_p} &= \sigma[\frac{W^*}{\sqrt{|\mathcal{N}_u| + 1}} \cdot (\frac{x_i}{\sqrt{|\mathcal{N}_i|}} + \frac{x_j}{\sqrt{|\mathcal{N}_j|}} + \cdots + \frac{x_k}{\sqrt{|\mathcal{N}_k|}}  \\ & \quad+ \frac{x_p}{\sqrt{|\mathcal{N}_p|+1}})] \\
				&= \sigma[\frac{|\mathcal{N}_u| \cdot W^* \cdot \mu(Y_u)}{\sqrt{(|\mathcal{N}_u| + 1) \cdot \!<\!d\!>\!}} + \frac{W^* \cdot \mu(Y_p)}{\sqrt{(|\mathcal{N}_u| + 1) \cdot (|\mathcal{N}_p|+1)}}] \\ \\
				h_{u_q} &= \sigma[\frac{W^*}{\sqrt{|\mathcal{N}_u| + 1}} \cdot (\frac{x_i}{\sqrt{|\mathcal{N}_i|}} + \frac{x_j}{\sqrt{|\mathcal{N}_j|}} + \cdots + \frac{x_k}{\sqrt{|\mathcal{N}_k|}}  \\ & \quad+ \frac{x_q}{\sqrt{|\mathcal{N}_q|+1}})]\\
				&= \sigma[\frac{|\mathcal{N}_u| \cdot W^* \cdot \mu(Y_u)}{\sqrt{(|\mathcal{N}_u| + 1) \cdot \!<\!d\!>\!}} + \frac{W^* \cdot \mu(Y_q)}{\sqrt{(|\mathcal{N}_u| + 1) \cdot (|\mathcal{N}_q|+1)}}]
			\end{aligned}
		\end{equation} 
		Since ${\rm SIM}(u,p) > {\rm SIM}(u,q)$, if we further utilize their features to characterize the similarities, we can transform it to  ${\rm SIM}(\mu(Y_u),\mu(Y_p)) > {\rm SIM}(\mu(Y_u),\mu(Y_q))$. As a result, in (\ref{eq:gen_huv}), we can consider that, a large part of the second term of $h_{u_p}$ can be combined with the first term, while only a smaller part of the second term $h_{u_q}$ can be combined with the first term. Similarly, after multiplying with the diagonal-like optimal weight matrix $W^*$ and the adjustment of monotonic softmax function, in the $Y_u$-th index of the latest outputs, the value of $u_p$ will larger than $u_q$, that is $h_{u_p}[Y_u] > h_{u_q}[Y_u]$. From the above, we can obtain that the adversarial nodes which are dissimilar with node $u$ will influence the aggregation of node $u$ more than those that are similar with node $u$, as the latter has a smaller probability in the $Y_u$-th class than the former.
		
	\end{proof}

\end{document}